\documentclass[journal,twoside,web]{ieeecolor}
\usepackage{generic}
\usepackage{cite}
\usepackage{amsmath,amssymb,amsfonts}
\usepackage{graphicx}
\usepackage{textcomp}
\usepackage{hyperref}

\usepackage{amsthm}
\usepackage{algorithm}
\usepackage{algpseudocode}

\usepackage{multirow}
\usepackage{array}
\usepackage{caption}
\usepackage{booktabs}
\usepackage{tabularx}
\usepackage{makecell}
\usepackage{subcaption}

\usepackage{rotating}

\makeatletter
\g@addto@macro\normalsize{%
  \setlength{\abovedisplayskip}{4pt plus 1pt minus 2pt}%
  \setlength{\belowdisplayskip}{4pt plus 1pt minus 2pt}%
  \setlength{\abovedisplayshortskip}{2pt plus 1pt minus 1pt}%
  \setlength{\belowdisplayshortskip}{3pt plus 1pt minus 1pt}%
}
\makeatother

\newcommand{\stat}[2]{%
  \makebox[2.7em][l]{#1}%
  \makebox[3.3em][l]{(#2)}%
}

\newcommand{\statbf}[2]{%
  \makebox[2.7em][l]{\textbf{#1}}%
  \makebox[3.3em][l]{\textbf{(#2)}}%
}

\algrenewcommand\algorithmicrequire{\textbf{Input:}}
\algrenewcommand\algorithmicensure{\textbf{Output:}}

\theoremstyle{definition}
\newtheorem{definition}{Definition}[section]

\theoremstyle{plain}
\newtheorem{lemma}[definition]{Lemma}
\newtheorem{theorem}[definition]{Theorem}

\theoremstyle{definition}
\newtheorem{remark}[definition]{Remark}

\def\BibTeX{{\rm B\kern-.05em{\sc i\kern-.025em b}\kern-.08em
    T\kern-.1667em\lower.7ex\hbox{E}\kern-.125emX}}
\begin{document}
\title{
From Real-Time Planning to Reliable Execution:\\Scalable Coordination for Heterogeneous Multi-Robot Fleets in Industrial Environments
}
% \author{First A. Author, \IEEEmembership{Fellow, IEEE}, Second B. Author, and Third C. Author, Jr., \IEEEmembership{Member, IEEE}
% % \thanks{This p }
% % \thanks{The }
% }
\author{Bo Cao
}

\maketitle

\begin{abstract}
With the increasing deployment of heterogeneous robot fleets in industrial environments, efficient coordination remains a critical challenge. Real-time path planning must simultaneously accommodate high robot densities and heterogeneous motion capabilities, while communication delays, execution uncertainties, and other disturbances may cause robots to deviate from the temporal assumptions underlying planned paths. Such deviations can lead to excessive waiting and congestion propagation across the fleet. This paper presents SCALE, a reactive online coordination framework that enables real-time planning while maintaining robust execution. Within this framework, we introduce a motion-induced conflict reduction mechanism to support the online generation of feasible paths for online conflict resolution. To mitigate the effects of disturbances, we further design a generalized Conjugate Action-Precedence Hypergraph (CAPH) that adaptively adjusts precedence relations among robots. Extensive validation experiments, together with a three-day deployment in a warehouse, demonstrate the practical feasibility and effectiveness of the proposed method.
\end{abstract}


\section{Introduction}
\label{sec:introduction}
\IEEEPARstart{I}{ntelligent} manufacturing systems and warehousing logistics are increasingly operated by unmanned robot fleets, including automated guided vehicles (AGVs) such as Kiva-like robots, autonomous forklifts, mobile manipulators, and large stacker robots. Compared with homogeneous mobile-robot teams, such fleets differ substantially in footprint and kinematic behavior. Coordinating them in dense workspaces is therefore not merely a path-planning problem: the system must generate feasible routes in real time, avoid collisions and deadlocks in limited space, and maintain robust execution despite communication latency, tracking errors, and uncertain waiting caused by humans or third-party facilities.

In today's industrial multi-robot deployments, coordination is commonly built on well-designed topological roadmaps or traffic zones. Collision avoidance and deadlock prevention are typically implemented by engineering-oriented methods such as zone control~\cite{zone-control}, Petri-net supervisors~\cite{petri-net}, glued-node reservation~\cite{glued-node}, or other handcrafted traffic rules. These methods are useful in specific scenarios because their safety logic is explicit and easy to certify. However, when dozens of tasks are involved, scheduling can take tens of seconds, which is not efficient enough for industrial systems. More importantly, for large-scale robot fleet, such strategies tend to be conservative: zones, guidepath segments, or reserved vertices often have to be allocated according to worst-case occupancy, reducing space utilization and degrading system efficiency in high-density scenarios~\cite{glued-node}. Their performance also strongly depends on manually designed layouts and rules, which severely limits their generality.

Recent advances in multi-agent path finding (MAPF) offer the possibility of more general planning. Representative algorithms, such as MAPF-LNS2~\cite{mapf-lns2}, LaCAM, and its improved variant LaCAM*~\cite{lacam, lacam_star}, have demonstrated strong scalability and solution quality in MAPF benchmarks. Nevertheless, most MAPF solvers rely on a graph-agent abstraction, in which each agent occupies a vertex or traverses an edge at a discrete time step. This creates a substantial gap when applied to real-world heterogeneous robots: the occupied space of a robot is determined not only by its vertex position, but also by its rigid-body footprint, orientation, and kinematic constraints. Although some approaches~\cite{large-agent-li,CCBS-1} partially relax the point-agent assumption, they still struggle to preserve the real-time capability and scalability of the original algorithms.
Moreover, in industrial environments, common disturbances, including occasional communication failures and human--robot interactions, can cause path execution to deviate from nominal discrete timing, compromising both efficiency and safety.
Robust MAPF methods reserve temporal margins~\cite{k-robut} or adjust precedence relations during execution~\cite{ADG,SADG}.
However, these methods still characterize inter-robot interactions primarily through pairwise vertex-occupancy conflicts or simplified footprint representations, limiting their applicability to heterogeneous robots.

To address the challenges above, we present \textbf{SCALE}, a \textbf{S}calable \textbf{C}oordination \textbf{A}rchitecture with \textbf{La}tency-resilient \textbf{E}xecution for large-scale heterogeneous robot fleet. It adopts an online reactive planning$\mbox{-}$execution architecture that groups potentially conflicting robots, generates locally feasible paths, and releases safe path segments through an execution authorization module under external disturbances.
The main contributions of this work are as follows: \textbf{1)} a generalized coordination framework for large-scale heterogeneous robot fleet; \textbf{2)} a motion-induced conflict reduction mechanism that enables real-time, kinematically feasible multi-robot path planning (MRPP); \textbf{3)} a generalized structure called Conjugate Action-Precedence Hypergraph (CAPH) for robust and latency-resilient asynchronous execution; and \textbf{4)} extensive simulations and real-world deployments in industrial scenarios with AGVs and forklifts, demonstrating sustained collision- and deadlock-free operation and a more than 30\% improvement in throughput.
Compared with traditional scheduling methods that separate route planning from execution control, SCALE models fleet-level path planning and asynchronous cooperative execution as an integrated online process, enabling scalable planning and improving while executing under unavoidable industrial uncertainties.

\begin{figure}[htbp]
    \centering
    \includegraphics[width=0.95\linewidth]{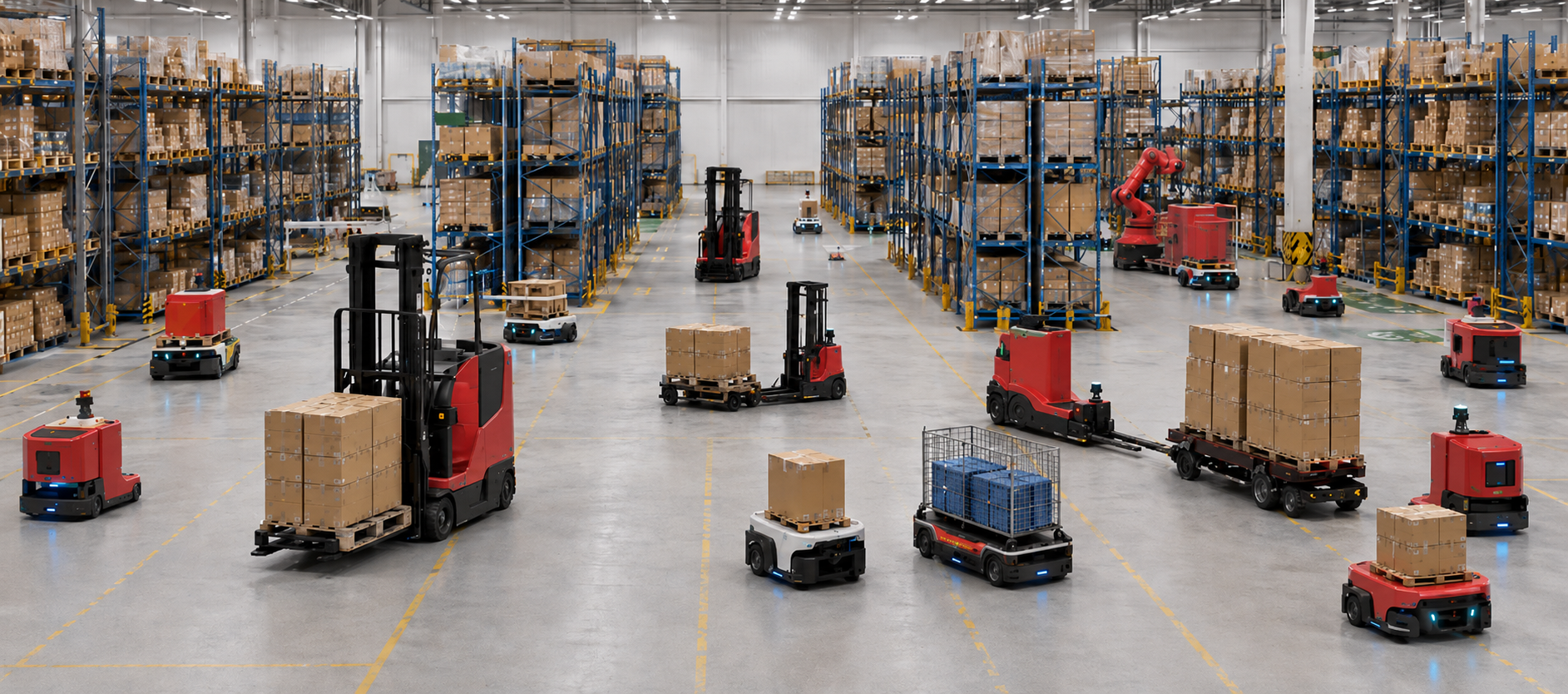}
    \caption{A typical industrial warehouse involving heterogeneous robot fleet.}
    \label{fig:scen}
\end{figure}

\section{Related Work}
\subsection{Scheduling of Industrial Multi-Robot Fleets}
Multi-robot scheduling has been extensively studied. Existing industrial approaches often rely on handcrafted heuristics tailored to specific scenarios. For example, \cite{pure-rule-agvs} adopts purely rule-based scheduling, while \cite{petri-net,zone-control} require predefined constraints under particular road-network structures. These methods can be efficient but are limited in generality. \cite{glued-node} provides a unified representation of inter-robot conflicts, and \cite{CAEH} proposes a more general unified scheduling framework; however, both lack scalability. Recently, MAPF-based methods have shown strong potential for coordination of large-scale robot fleet. \cite{well-formed,RHCR} enable real-time planning for hundreds of physical robots, but they typically simplify robot dynamics and may require structured roadmaps or well-formed conditions to ensure feasible planning. In contrast, the paper proposes a general and scalable lifelong scheduling framework for heterogeneous robot fleet.

\subsection{Multi-Robot Path Planning}
Multi-robot path planning provides robot fleet with a degree of general planning capability. Recent progress in MAPF under the point-agent assumption has substantially improved the scalability of planning algorithms. For example, \cite{mapf-lns2} employs large neighborhood search to enhance the scalability of classical search-based methods such as \cite{ECBS,pbs,CBS}. Another line of work, including \cite{pibt,lacam,lacam_star}, seeks feasible solutions through priority inheritance among agents, enabling planning for tens of thousands of agents within seconds. Several studies, such as \cite{large-agent-li,hetero-mapf,2022CBS-dynamics}, have attempted to incorporate more realistic assumptions into discrete planning, while~\cite{db-lacam} extends LaCAM to continuous motion-space search. However, they cannot perform real-time planning, which is unacceptable in industrial applications. Some of them also adopt overly idealized robot models, such as representing robots of different sizes as disks, leading to a substantial gap between their assumptions and real-world physical constraints.

\subsection{Robust Execution of Planned Paths}
For robust collaborative path execution, some approaches reserve temporal slack during planning so that bounded delays can be safely accommodated, as in $k$-robust and $p$-robust MAPF~\cite{k-robut,p-robut}. In~\cite{STN-ma}, stable execution of multi-robot plans is achieved by optimizing robot velocities.These formulations grow increasingly conservative as the delay margin increases and are generally inadequate when the system occasionally experiences extreme crashes or large delays.
Other works focus on precedence relations among actions of robots. For example, ADG~\cite{ADG} enforces a fixed ordering of execution admission, while time-independent planning works on a globally synchronized clock~\cite{TTP}. More recently, SADG~\cite{SADG} extends ADG with switchable dependencies, adjusting the execution order to reduce unnecessary waiting. In addition, methods such as~\cite{pie_d} couple planning and execution, where a committed path segment is executed while the remaining plan is improved online. However, these methods still primarily reason over pairwise topological dependencies or simplified vehicle models. For heterogeneous robots, a more realistic and efficient method is required to handle unavoidable disturbances.

% ==================================================================================
\section{Coordinating Heterogeneous Robot Fleets}
\label{sec:problem_formulation}

We formulate the coordination of heterogeneous robot fleet from two coupled perspectives.
The first is \emph{fleet-level path planning}, which computes a kinematically feasible path for each robot. The second is \emph{robust path execution}, which regulates the
release of the planned paths through action-level precedence constraints. 
Let the shared coordination space be represented by an undirected graph $G=(V,E)$, where
$V$ is a finite set of vertices and $E\subseteq V\times V$ is the set of topological
adjacencies. The robot fleet is denoted by $\mathcal{R}=\{r_1,r_2,\ldots,r_N\}$. Each robot
$i\in\mathcal{R}$ has a type $\tau_i\in\mathcal{T}$, which determines physical properties like its footprint and turning characteristics. For robot $i$, a path is defined as
\[
    p_i=
    \bigl((v_i^1,\rho_i^1),(v_i^2,\rho_i^2),\ldots,(v_i^{L_i},\rho_i^{L_i})\bigr)
    \in (V\times\mathbb{N})^{L_i},
\]
where $v_i^k\in V$ is the $k$-th vertex and $\rho_i^k\in\mathbb{N}$ is its nominal
precedence rank. A smaller rank indicates earlier execution. Unlike the discrete
time index in step MAPF, $\rho_i^k$ is not a discrete timestamp but an ordering over vertices.
Every consecutive vertex pair for a robot corresponds to an edge transition: $\{v_i^k,v_i^{k+1}\}\in E$ for all $k=1,\ldots,L_i-1$. If $p_i$ satisfies the kinematics constraints of
robot $i$, it induces an action sequence
$\sigma_i=(a_i^1,a_i^2,\ldots,a_i^{L_i-1})$, where
$a_i^k=(i,k,v_i^k,v_i^{k+1},\gamma_i^k)$ denotes the continuous motion from $v_i^k$ to
$v_i^{k+1}$. Here $\gamma_i^k(\lambda)\rightarrow SE(2)$ denotes the trajectory corresponding to the motion, and $\lambda\in[0,1]$ is a normalized parameter. 
The fleet-level path plan is $\mathcal{P}:=\{p_i\}_{i\in\mathcal{R}}$, and the induced action set is $\mathcal{A}(\mathcal{P}):=\bigcup_{i\in\mathcal{R}}\{a_i^k\mid k=1,\ldots,L_i-1\}$.
% For a robot of type $\tau_i$, the trajectory must belong to the admissible
% motion set $\Gamma_{\tau_i}(v_i^k,v_i^{k+1})$, which contains all type-feasible trajectories connecting $v_i^k$ to $v_i^{k+1}$ under the robot's kinematic constraints.

\subsection{Heterogeneous Multi-Robot Path Planning}
\label{subsec:heterogeneous_mrpp}

Given start vertices $\mathbf{v}^{\mathrm{s}}=(v_1^{\mathrm{s}},\ldots,v_N^{\mathrm{s}})$
and goal vertices $\mathbf{v}^{\mathrm{g}}=(v_1^{\mathrm{g}},\ldots,v_N^{\mathrm{g}})$,
the heterogeneous multi-robot path planning problem is to find a plan
$\mathcal{P}=\{p_i\}_{i\in\mathcal{R}}$ such that
$v_i^1=v_i^{\mathrm{s}}$ and $v_i^{L_i}=v_i^{\mathrm{g}}$ without any conflicts. Each path must be rank-consistent and feasible under the conflict model defined below.

Classical MAPF considers two topological conflicts: vertex conflicts and swap conflicts.
Following this convention, for two actions $a_i^k$ and $a_j^\ell$ with $i\neq j$, we set
$\operatorname{TC}(a_i^k,a_j^\ell)=1$ if either
$v_i^{k+1}=v_j^{\ell+1}$, which denotes two actions entering the same vertex, or
$v_i^k=v_j^{\ell+1}$ and $v_i^{k+1}=v_j^\ell$, which denotes two actions traversing the
same edge in opposite directions. Obviously, topological conflicts are insufficient for robots. Let $\mathcal{B}_{\tau_i}\subset\mathbb{R}^2$ denote the body-frame footprint of robot type
$\tau_i$. For action $a_i^k$, let
$\gamma_i^k(\lambda)=(x_i^k(\lambda),y_i^k(\lambda),\theta_i^k(\lambda))$ be its
continuous pose, and let $R(\theta)$ be the planar rotation matrix. The swept
region of $a_i^k$ is defined as
\begin{equation}
    \mathcal{W}_{\epsilon}(a_i^k)
    :=
    \bigcup_{\lambda\in[0,1]}
    \left(
        R(\theta_i^k(\lambda))\mathcal{B}_{\tau_i}
        +
        \begin{bmatrix}
            x_i^k(\lambda)\\
            y_i^k(\lambda)
        \end{bmatrix}
    \right)
    \oplus
    \mathbb{B}_{\epsilon},
    \label{eq:swept_region}
\end{equation}
where $\oplus$ denotes the Minkowski sum and
$\mathbb{B}_{\epsilon}=\{\pi \in\mathbb{R}^2\mid \|\pi\|_2\leq\epsilon\}$ is a
safety margin with radius $\epsilon\geq0$. A motion-induced conflict is defined as
$
    \operatorname{MC}(a_i^k,a_j^\ell)=1
    \iff
    \mathcal{W}_{\epsilon}(a_i^k)\cap\mathcal{W}_{\epsilon}(a_j^\ell)\neq\emptyset,
    i\neq j.
$
This definition captures collisions caused by robots with different footprints and nonholonomic constraints. The overall conflict model is
\(\operatorname{C}(a_i^k,a_j^\ell)=1\) if
\(\operatorname{MC}(a_i^k,a_j^\ell)=1\) or
\(\operatorname{TC}(a_i^k,a_j^\ell)=1\).
% We define the overall conflict indicator between two actions as
% $\operatorname{C}(a_i^k,a_j^\ell)=1$ if
% $\operatorname{TC}(a_i^k,a_j^\ell)=1$ or
% $\operatorname{MC}(a_i^k,a_j^\ell)=1$. The objective of the execution layer is then to
% assign precedence relations to all conflicting action pairs so that no conflicting actions
% are released simultaneously.

\subsection{Hybrid Action-Precedence Graph for MRPP}
\label{ADG}
Due to asynchronous execution and external disturbances in practice, the planned paths cannot always be executed in strict accordance with the discrete rank ordering.
Drawing on the action-dependency in ADG~\cite{ADG}, we introduce a generalized representation, termed Hybrid Action-Precedence Graph (HAPG). It defines action-level dependencies based on physical interactions among robots rather than vertex occupancy.

\begin{definition}[Hybrid Action-Precedence Graph]
\label{def:HAPG}
Given a set of paths $\mathcal{P}$ and its induced action set
$\mathcal{A}$, the HAPG is a directed graph
$\mathcal{G}_a=(\mathcal{A},\mathcal{E})$, where
$\mathcal{E}\subseteq\mathcal{A}\times\mathcal{A}$. A directed edge
$(b,c)\in\mathcal{E}$ indicates that $c$ can be executed only after $b$ has been completed.
There are two kinds of edges:
$\mathcal{E}=\mathcal{E}_{\mathrm{self}}\cup\mathcal{E}_{\mathrm{inter}}$, where
\[
\begin{aligned}
\mathcal{E}_{\mathrm{self}}
&:=
\{(a_i^k,a_i^{k+1})\mid i\in\mathcal{R},\ k=1,\ldots,L_i-2\}, \\
\mathcal{E}_{\mathrm{inter}}
&:=
\{(a_i^k,a_j^\ell)\mid i\neq j,\ 
\operatorname{C}(a_i^k,a_j^\ell)=1,\ 
a_i^k\prec_{\rho}a_j^\ell\}.
\end{aligned}
\]
Here $\mathcal{E}_{\mathrm{self}}$ preserves the path ordering of each robot,
whereas $\mathcal{E}_{\mathrm{inter}}$ assigns precedence relations to conflicting actions. Specifically,
$a_i^k\prec_{\rho}a_j^\ell$ if $\rho_i^k\leq\rho_j^\ell \wedge \operatorname{C}(a_i^k,a_j^\ell)=1$. See Algorithm~\ref{alg:hapg_construction} for details.
\end{definition}

Let
\(\{\mathrm{staged},\mathrm{in\mbox{-}progress},\mathrm{completed}\}\)
be the state set. The action state is initialized as \(\xi(a)\!=\!\mathrm{staged}\).
For \(a\), define
\(\operatorname{Pred}(a)\!=\!\{a'\!\in\!\mathcal{A}\mid(a',a)\!\in\!\mathcal{E}\}\).
The action can enter \(\mathrm{in\mbox{-}progress}\) only when every
\(a'\!\in\!\operatorname{Pred}(a)\) satisfies
\(\xi(a')\!=\!\mathrm{completed}\), and turns to \(\mathrm{completed}\)
after the action is finished.

\begin{algorithm}[htbp]
\caption{Hybrid Action-Precedence Graph Construction}
\label{alg:hapg_construction}
\begin{algorithmic}[1]
\Require Fleet-level path plan $\mathcal{P}=\{p_i\}_{i\in\mathcal{R}}$
\Ensure Hybrid action-precedence graph $\mathcal{G}_a=(\mathcal{A},\mathcal{E})$

\State $\mathcal{A}\gets\emptyset$, $\mathcal{E}\gets\emptyset$

\ForAll{$i\in\mathcal{R}$}
    \For{$k=1$ \textbf{to} $L_i-1$}
        \State $a_i^k=(i,k,v_i^k,v_i^{k+1},\gamma_i^k)$
        \State $\mathcal{A}\gets\mathcal{A}\cup\{a_i^k\}$
        \If{$k>1$}
            \State $\mathcal{E}\gets\mathcal{E}\cup\{(a_i^{k-1},a_i^k)\}$
        \EndIf
    \EndFor
\EndFor

\ForAll{$a_i^k,a_j^\ell\in\mathcal{A}$ with $i<j$}
    \If{$\operatorname{C}(a_i^k,a_j^\ell)=1$}
        \State  $\mathcal{E}$ $\leftarrow$ $(a_i^k,a_j^\ell)$  if
        $a_i^k\prec_{\rho}a_j^\ell$; otherwise $\mathcal{E}$ $\leftarrow$ $(a_j^\ell,a_i^k)$
    \EndIf
\EndFor

% \State $\mathcal{G}_a\gets(\mathcal{A},\mathcal{E})$
% \State \Return $\mathcal{G}_a$
\end{algorithmic}
\end{algorithm}

\begin{lemma}
\label{lem:hapg-safety-deadlock}
Let \(\mathcal P\) be a feasible rank-consistent MRPP plan. If each action is
admitted only after all its HAPG predecessors are completed, then the HAPG
execution is collision-free and deadlock-free.
\end{lemma}
\begin{proof}
For any two inter-robot actions \(b\) and \(c\), \(\operatorname C(b,c)=0\) indicates the absence of collision. If
\(\operatorname C(b,c)=1\), HAPG includes an inter-robot precedence edge between
them: one action can start only after the other has completed, so the
conflicting swept regions never overlap in time. The self edges serialize the
actions of each robot. So the execution is collision-free.
Rank consistent paths induces a strict ordering among all actions. Therefore every edge in \(\mathcal E\) points from an earlier action to a later action in this order. As long as the path plan is feasible, no action can precede itself in the induced temporal ordering. Thus HAPG is acyclic, which means it is deadlock-free.
\end{proof}

\begin{figure*}[htbp]
    \centering
    \includegraphics[width=\linewidth]{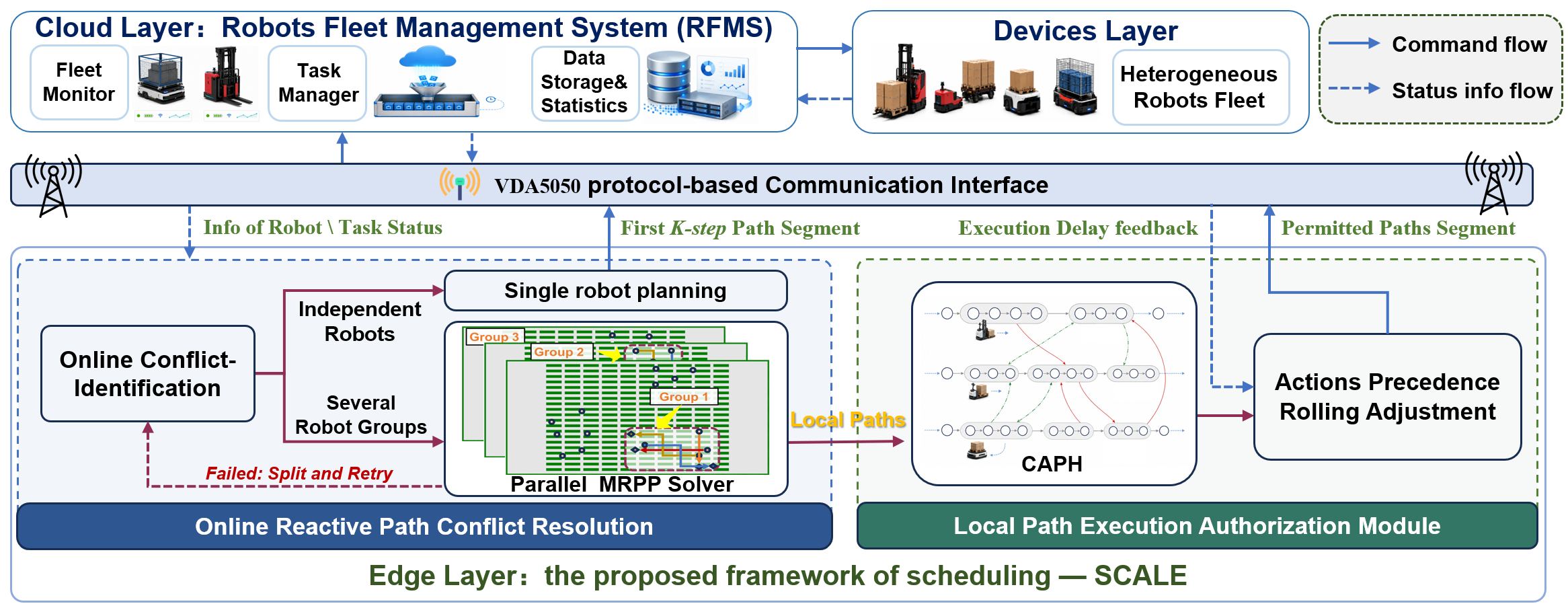}
    \caption{Overall coordination architecture of the proposed method.}
    \label{fig:framework}
\end{figure*}

\begin{figure}[htbp]
    \centering
    \includegraphics[width=\linewidth]{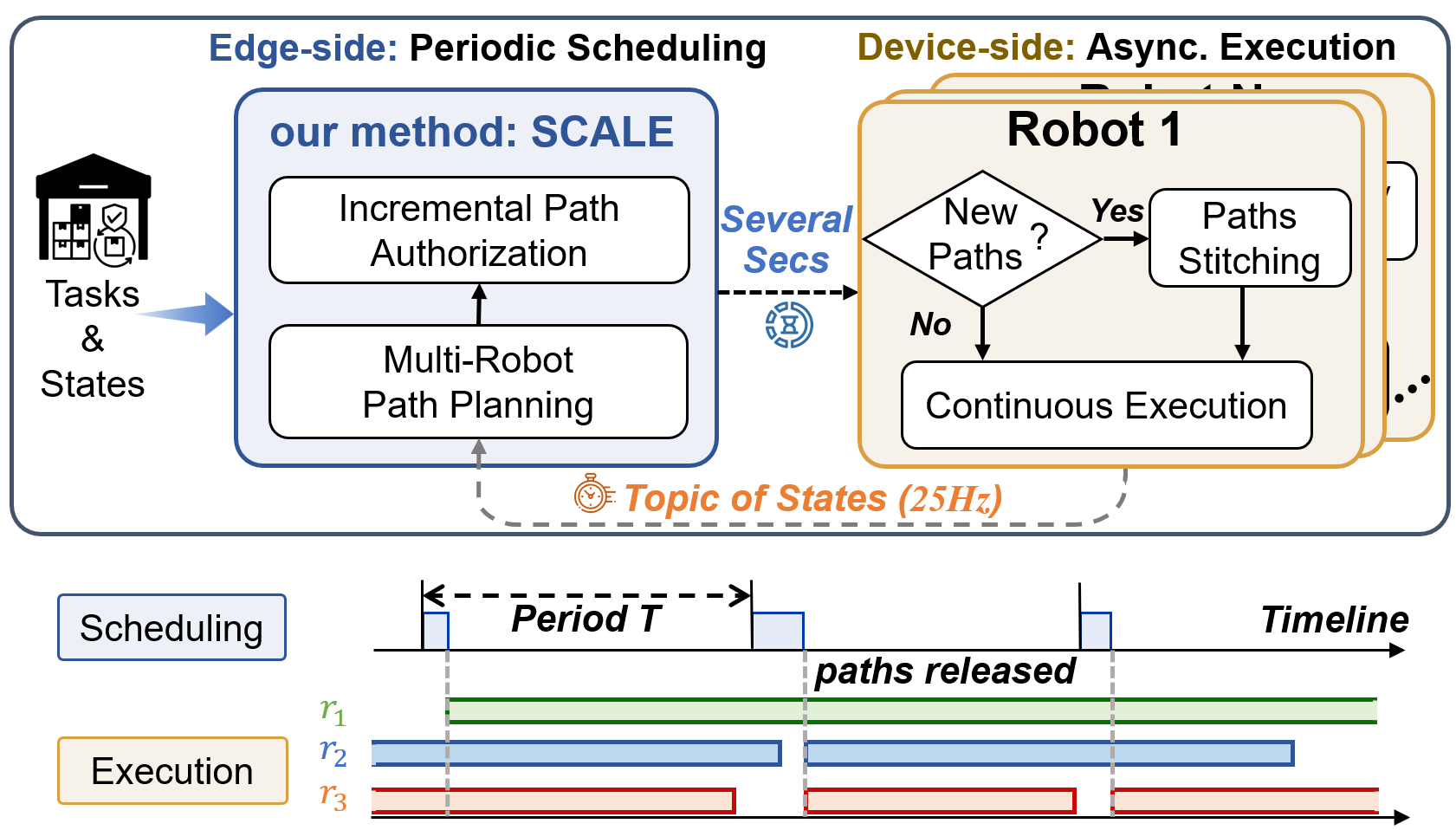}
    \caption{The scheduling process of industrial robot fleet. Top: scheduling and execution workflow. Bottom: an example of asynchronous execution of 3 robots over time, where the time consumption of scheduling has negligible effect on continuous execution.}
    \label{fig:flow_chart}
\end{figure}

\section{Main Approaches}
The overall coordination framework proposed is illustrated in Fig.~\ref{fig:framework}. It consists of two main modules: an online reactive path planning module and a resilience-guaranteed path execution module. The former generates feasible path sequences for each robot over a graph-based road network, while the latter determines the subsequent paths authorization locally.

In industrial robot fleet scheduling, paths generated by edge-side devices are not dispatched to robots as complete routes. Dynamic task arrivals, heterogeneous kinematic behaviors, and asynchronous execution make long-horizon path commitment difficult to control in practice. We therefore adopt a bounded path-release strategy: in each scheduling round, at most \(K\) steps of the planned path are committed to each robot, where \(K\) is an empirical commitment horizon. A released prefix is required to be executable to its end without collision or deadlock. For candidate local groups whose paths are bounded by this horizon, the execution authorization layer further uses CAPH in Sec.~\ref{subsec:caph} to release paths under action-precedence constraints. The next planning step starts from the terminal vertex of the last released path for each robot. This strategy, illustrated in Fig.~\ref{fig:flow_chart}, serves as the basis for efficient and controllable scheduling of a large-scale industrial robot fleet.

\subsection{Online Reactive Path Conflict Resolution}
\label{sec:path_planning}
% The planning layer therefore focuses on the unreleased suffixes of existing path plans. Before a new bounded prefix is committed, the planner identifies whether these suffixes contain local interactions that should be resolved collaboratively.
\subsubsection{Potential Conflict Robot Group Identification}
The online planner does not replan all robots jointly in every round. It first extracts robot groups whose unreleased paths may require repair to handle local path conflicts, while the remaining independent robots follow single-robot planning. Let \(\mathcal R_{\mathrm{idle}}\) denote idle robots and \(\mathcal R_{\mathrm{act}}=\mathcal R\setminus\mathcal R_{\mathrm{idle}}\) denote robots with active tasks. For each \(i\in\mathcal R_{\mathrm{act}}\), let \(\bar p_i\) be the unreleased segment of the path planned in the previous scheduling round. Only the first \(K\) actions induced by \(\bar p_i\) are considered, forming the horizon action set \(\mathcal A_{i,K}\); if \(\bar p_i\) is shorter than \(K\), its terminal vertex is treated as persistent occupation.
Thus, the conflict source is the path information that has already been planned but has not yet been released for execution. The horizon action sets inherit partial execution dependencies from the CAPH constructed in the previous scheduling process, too. To determine whether all the unreleased actions may block each other, we incrementally augment the dependency structure using the currently planned paths of all robots, rather than considering each tentative group in isolation, and then check whether a directed deadlock cycle is formed. Let \(\mathfrak C_{K}\) denote the set of detected cycles, as defined in Sec.~\ref{subsec:caph}. For a cycle \(Z\in\mathfrak C_{K}\), define the involved robot set as
\(
\mathcal R(Z):=\{\,i\in\mathcal R \mid \exists a_i^k\in Z\,\}.
\)
The initial groups are obtained by merging all overlapping cycle-induced sets:
\(
\mathbb G=\operatorname{Merge}\bigl(\{\mathcal R(Z)\mid Z\in\mathfrak C_{K}\}\bigr),
\)
where \(\operatorname{Merge}(\cdot)\) operates on robot sets: two cycle-induced sets are unioned whenever they share at least one robot, i.e., \(\mathcal R(Z_1)\cap\mathcal R(Z_2)\neq\emptyset\). Hence, multiple deadlock cycles coupled through the same robot are assigned to one local planning group. Let \(v_i^{\mathrm{cur}}\) be the current occupied roadmap vertex of robot \(i\), and let \(\bar a_i(v)\) denote its zero-motion occupation at vertex \(v\). The same horizon is then checked against idle robots. If an idle robot \(j\) satisfies \(\operatorname{C}(a_i^k,\bar a_j(v_j^{\mathrm{cur}}))=1\) for any \(a_i^k\in\mathcal A_{i,K}\), it is added to the group containing \(i\).

Once a group \(\mathcal S \in \mathbb G\) has been identified, local goals are assigned to guide the robots. Let \(\mathcal U_{\mathcal S}\) denote the vertices set containing only the path vertices of the cyclic actions in \(\mathcal S\):
\[
\begin{aligned}
\mathcal U_{\mathcal S}
:=
\bigcup_{\substack{
a_i^k(v_i^k,v_i^{k+1})\in Z, \quad Z\in\mathfrak C_{K}
}}
\{\,v_i^k,v_i^{k+1}\,\}.
\end{aligned}
\]
The local coordination region is obtained by breadth-first search on the roadmap graph:
\(
\mathcal V_{\mathcal S}
=
\operatorname{BFS}_{K'}\!\left(\mathcal U_{\mathcal S}\right),
\)
where \(K'\) is the depth of expansion.
For an active robot, the task goal is directly used when
\[
\Phi_i:=
\left[
v_i^{\mathrm g}\in\mathcal V_{\mathcal S}
\ \wedge\
\operatorname{C}(\bar a_i(v_i^{\mathrm g}),\bar a_j(v_j^{\mathrm{cur}}))=0,\
\forall j\in\mathcal S\setminus\{i\}
\right],
\]
otherwise the local goal is selected along its current single-robot path. Define two selection predicates:
\[
\begin{aligned}
\operatorname{Goal}_i(v)
&:=
\bigl[
v\notin\mathcal V_{\mathcal S}
\ \wedge\
\operatorname{C}(\bar a_i(v),\bar a_j(v_j^{\mathrm{cur}}))=0,\\
&\quad
\forall j\in\mathcal S\setminus\{i\}
\bigr],\\
\operatorname{Free}_i(v)
&:=
\bigl[
v\notin\mathcal V_{\mathcal S}
\ \wedge\
\operatorname{C}(\bar a_i(v),a_m^q)=0,\\
&\quad
\forall m\in\mathcal S\cap\mathcal R_{\mathrm{act}},\
\forall a_m^q\in\mathcal A_{m,K}
\bigr].
\end{aligned}
\]
Then the strategy for local goals assignment is
\[
\begin{aligned}
g_i^{\mathrm{loc}}
=
\begin{cases}
v_i^{\mathrm g},
& i\in\mathcal R_{\mathrm{act}}\ \mathrm{and}\ \Phi_i,\\[1mm]
\operatorname{First}_{p_i^{\mathrm{sr}}}\{v\in V\mid \operatorname{Goal}_i(v)\},
& i\in\mathcal R_{\mathrm{act}}\ \mathrm{and}\ \neg\Phi_i,\\[1mm]
\displaystyle
\underset{\substack{v\in V\\ \operatorname{Free}_i(v)}}{\arg\min}
\quad l(v_i^{\mathrm{cur}},v),
& i\in\mathcal R_{\mathrm{idle}}.
\end{cases}
\end{aligned}
\]
Here \(\operatorname{Goal}_i(v)\) indicates that vertex \(v\) is outside the cycle-expanded region and can be statically occupied by robot \(i\) without conflicting with the current positions of all other robots in the group, whereas \(\operatorname{Free}_i(v)\) indicates that an idle robot can occupy \(v\) without conflicting with the active horizon actions. In addition, \(p_i^{\mathrm{sr}}\) is the current single-robot path toward the task goal, \(\operatorname{First}_{p_i^{\mathrm{sr}}}(\cdot)\) returns the first vertex along this path satisfying the specified condition, and \(l(\cdot,\cdot)\) is the Manhattan distance on the roadmap. This assignment keeps active robots moving toward their task destinations while providing the local MRPP solver with local goals for conflict resolution; idle robots are assigned nearby avoidance vertices to avoid unnecessary blocking.

\subsubsection{Independent Robot Path Planning}
Robots that are not assigned to any group are planned independently. This single-robot planning step serves as a lightweight path planning and congestion-avoidance procedure for robots whose current decisions have no immediate influence on others.
Inspired by traffic-flow-based optimization~\cite{traffic-flow}, we construct a guidance graph from the current planned paths, so that independent planning can account for the near-future congestion. Since paths are released under action-precedence constraints, only the path segments corresponding to actions that are not yet $\mathrm{completed}$ should contribute to the guidance cost. Accordingly, for each robot path \(p_i\in\mathcal P_t\), let \(\widehat p_i^{\,t}\subseteq E\) denote the relevant set of edges. 
The flow count of \((\mu,\nu)\in E\) is:
\begin{equation}
    \varrho(\mu,\nu)
    =
    \sum_{p_i\in\mathcal P}
    \mathbf{1}\{(\mu,\nu)\in\widehat p_i\},
\end{equation}
which represents the number of paths that pass through \((\mu,\nu)\). The vertex-level congestion and contraflow weight are estimated as
\(
\kappa(\nu)=\sum_{\mu\in\mathcal N(\nu)} \varrho(\mu,\nu), \quad
\chi(\mu,\nu)=\varrho(\mu,\nu)\varrho(\nu,\mu).
\)
These quantities are incorporated into the edge cost:
\(
c(\mu,\nu)=\bigl(\chi(\mu,\nu),\, d(\mu,\nu)+\beta \kappa(\nu)\bigr),
\)
where \(d(\mu,\nu)\) denotes the edge length and \(\beta>0\) is a scaling coefficient. We thus construct the guided graph
\(
G^{\mathrm g}=(V,E,c),
\)
on which A*~\cite{a_star} is applied. In this way, severe opposite-direction traffic is avoided with highest priority, while usual congestion is penalized as a secondary criterion. The path planned for an independent robot is not necessarily the shortest in distance, but rather a low-congestion path.

\subsubsection{Search-Integrated Conflict Discretization for Multi-Robot Path Planning}
\label{sec:SICD}
The motion-induced conflict model in Sec.~\ref{sec:problem_formulation} captures the physical interaction between heterogeneous robots through their swept
footprints. For an MRPP algorithm, directly evaluating this continuous
model during searching would require repeated pairwise
polygon-overlap evaluations over candidate actions, which is
inefficient for real-time planning. We introduce Search-Integrated
Conflict Discretization (SICD), a reduction mechanism that embeds
continuous-motion conflict detection into the discrete search process.
Instead of invoking a geometric collision detector for every
robot pair, it associates each candidate action with a graph-level
occupancy set and uses the incremental nature of search to discard
infeasible joint decisions.

Consider a single searching round in which the planner evaluates a
transition from the current fleet state $D^t$ to a tentative successor
state $D^{t+1}$. For robot $i$, the corresponding continuous action is
$a_i^t=(i,d_i^t,d_i^{t+1},\gamma_i^t)$, where
$d_i^t,d_i^{t+1}\in V$. Let $X(v)\in\mathbb{R}^2$ denote the planar
coordinate of vertex $v$. SICD maps the swept footprint of $a_i^t$ to a
occupancy set:
\begin{equation}
    \mathcal{O}(a_i^t)
    =
    \left\{
    v\in\mathcal{N}_G(\gamma_i^t)
    \;\middle|\;
    X(v)\in\widehat{W}_{\epsilon}(a_i^t)
    \right\},
    \label{eq:sicd_occupancy}
\end{equation}
where $\mathcal{N}_G(\gamma_i^t)$ denotes the trajectory-wise roadmap
neighborhood pre-generated along $\gamma_i^t$, and
$\widehat{W}_{\epsilon}(a_i^t)$ is the sampled occupied region of the
action.
In each search round, the planner incrementally evaluates all candidate one-step movements and determines whether they can be incorporated into the same state transition. Then we embed the occupancy set
$\mathcal{O}(a_i^t)$ of each action into a bitmap
$\mathbf{b}(a_i^t)\in\{0,1\}^{|V|}$, where each entry is a binary
occupancy variable and the component associated with vertex $v$ is one
if and only if $v\in\mathcal{O}(a_i^t)$. Let $\mathbf{B}$ denote the accumulated bitmap obtained from all actions that have already been accepted in the current round. Then the newly considered action $a_i^t$ is compatible if and only if
\( \left(\mathbf{B}\;\&\;\mathbf{b}(a_i^t)\right)=\mathbf{0}. \)
Here $\&$ denotes bitwise AND. If the condition holds, the accumulated
occupancy is updated by the bitwise OR operation
$\mathbf{B}\leftarrow\mathbf{B}\vee\mathbf{b}(a_i^t)$; otherwise the
current partial joint decision is identified as conflicting. Since the
bitmap entries are binary variables, both the intersection and the
update can be implemented by efficient word-level bit operations.
Therefore, SICD converts repeated continuous sweeping collision detection into incremental bitmap set-intersection while preserving the
footprint-awareness of the underlying kinematic model. It should be noted that the mechanism is designed for collision detection on
the roadmap graph during planning, while complementary geometric collision avoidance during continuous execution is enforced by the execution layer, as discussed in Sec.~\ref{subsec:caph}. Moreover, SICD is not tied to a particular solver and can be integrated into any discrete-time search-based MRPP method. The specific details of the mechanism are shown in Algorithm~\ref{alg:sicd}.

\begin{algorithm}[htbp]
\caption{SICD collision detection for a state transition}
\label{alg:sicd}
\begin{algorithmic}[1]
\Require Current fleet state $D^t$, tentative successor state
         $D^{t+1}$
\Function{SICD}{$D^t,D^{t+1}$}
\State $\mathbf{B}\gets \mathbf{0}$ \Comment{occupied discretized footprint bitmap}
\ForAll{$i\in\mathcal{R}$}
    \If{$D^{t+1}_i$ is undefined}
        \State \textbf{continue}
    \EndIf
    \State $a_i^t\gets (i,D^t_i,D^{t+1}_i,\gamma_i^t)$
    \State $\mathbf{b}_i\gets \mathbf{b}(a_i^t)$
    \If{$(\mathbf{B}\;\&\;\mathbf{b}_i)\neq \mathbf{0}$}
        \State \Return \textbf{true} \Comment{collision detected}
    \EndIf
    \State $\mathbf{B}\gets \mathbf{B}\vee\mathbf{b}_i$
\EndFor
\State \Return \textbf{false} \Comment{collision free}
\EndFunction
\end{algorithmic}
\end{algorithm}

\textbf{Complexity Analysis.}
Let \(w_m\) be the machine word size. A single SICD query performs one
bitmap intersection over \(\lceil |V|/w_m\rceil\) words; if the action
is accepted, one bitmap union of the same size is applied. Therefore,
one transition detection costs \(O(|V|/w_m)\cdot O(1)\), where the \(O(1)\) is the complexity of bitmap-level operation.
For a collective transition with $N$ robots, the worst-case round complexity
is \(O(N|V|/w_m)\).
By comparison, a direct polygon-overlap (PO) scheme repeatedly evaluates
continuous geometry for robot pairs and has
$O(N^2Lc_{\mathrm{poly}})$ worst-case complexity per search round, where
$N$ is the number of robots, $L$ is the number of footprint samples along
an action, and $c_{\mathrm{poly}}$ denotes the cost of one
polygon-overlap predicate. Since $c_{\mathrm{poly}}$ involves
floating-point geometric predicates on polygonal footprints, frequent PO
calls introduce substantial and redundant computation during incremental
searching. SICD reduces this burden by replacing pairwise continuous
overlap detections with one bitmap query against the accumulated
occupancy for each candidate action.

\begin{remark}
When local MRPP fails to find a feasible solution for a group within the time limit, the group is split. The planner then reduces the commitment horizon \(K\) and relaxes the timeout threshold before retrying the resolution process. This fallback strategy allows the online conflict-resolution process to iteratively pursue a feasible resolution.
\end{remark}

\subsection{Conjugate Action-Precedence Hypergraph}
\label{subsec:caph}
% (Planning and Improving while Executing)
% HAPG provides action-level execution constraints for heterogeneous MRPP plans. However, MRPP plans typically ignore external disturbances which often occur in industrial environments.

To achieve latency-resilient path execution among heterogeneous robots, we introduce the Conjugate Action-Precedence Hypergraph (CAPH), which accommodates delays caused by external uncertainty. It preserves the original action-level dependencies of HAPG while lifting selected inter-robot dependencies to a capsule-level representation, termed the Conjugate Precedence Capsule (CPC). This representation reduces unnecessary waiting through dependency switching.

\subsubsection{CPC Construction}
We first introduce the action capsule, the basic component of CAPH:
\begin{definition}[Action Capsule] \label{def:AC}
For robot \(i\), let
\(
\sigma_i=(a_i^1,\ldots,a_i^{L_i-1})
\)
be its action sequence. An action capsule is a contiguous action segment
\(
Q_i^{p:q}:=\langle a_i^p,\ldots,a_i^q\rangle,
\quad 1\le p\le q\le L_i-1 .
\)
% The first and last actions of \(Q_i^{p:q}\) are denoted by
% \(
% h(Q_i^{p:q})=a_i^p, t(Q_i^{p:q})=a_i^q .
% \)
For each action \(a_i^k\), let \(\Delta_i^k\) denote
its nominal traversal duration, estimated from the length of the
corresponding trajectory and the speed of robot type \(\tau_i\). The entry and exit times of an action capsule are denoted by \(T_s(\cdot)\) and \(T_g(\cdot)\), respectively.
Then a capsule
\(Q_i^{k:k+\ell}\) is associated with an entry time
\(\hat T_s(Q_i^{k:k}) \ge \hat T_g(Q_i^{\sim :k-1})\) and an exit time \(\hat T_g(Q_i^{k:k+\ell}) \ge T_s(Q_i^{k:k}) + \sum_{n=k}^{k+\ell} \Delta_i^n\). An action capsule has the same states as an action: \(\{\mathrm{staged},\mathrm{in\mbox{-}progress}, \mathrm{completed}\}\).
\end{definition}

The CPC construction process is based on the initial dependencies in HAPG. Each action is first wrapped as a singleton action capsule:
\(
\mathcal Q_0=
\{Q_i^{k:k}\mid i\in\mathcal R,\ k=1,\ldots,L_i-1\}.
\)
Let \(\psi(a_i^k)=Q_i^{k:k}\). This mapping preserves the
original HAPG dependencies:
\(
(\psi(a),\psi(b))
\Longleftrightarrow
(a,b)\in \mathcal{E} 
\).
For any pair of robot paths \((p_i,p_j)\), we define the conflict
signature of action \(a_j^\ell\) with respect to \(p_i\) as
\(
\Gamma_{i\leftarrow j}(\ell)
=
\{\,k\in\{1,\ldots,L_i-1\}\mid \operatorname{C}(a_i^k,a_j^\ell)=1\,\},
\ell=1,\ldots,L_j-1 .
\)
This signature characterizes the local conflict context of \(a_j^\ell\)
with respect to the reference path \(p_i\).
A non-empty interval \([s,t]\) on \(p_j\) is called a maximal constant-signature interval if
\(
\Gamma_{i\leftarrow j}(s)
=
\Gamma_{i\leftarrow j}(s+1)
=
\cdots
=
\Gamma_{i\leftarrow j}(t)
=
\mathcal I\neq\emptyset
\). 
This interval induces the capsule
\(Q_j^{s:t}=\langle a_j^s,\ldots,a_j^t\rangle\). Its conjugate capsule on
\(p_i\) is defined as the minimum contiguous span covering the signature
\(\mathcal I\), namely \(Q_i^{p:q}\) with \(p=\min\mathcal I\) and
\(q=\max\mathcal I\). If the conflict signatures inherited from HAPG between the two capsules
do not form bidirectional dependencies, a candidate capsule pair is formed, denoted by
\(
\omega_{ij}^{\zeta}:=(Q_i^{p:q},Q_j^{s:t}).
\)
where \(\zeta\) indexes the capsule pair.

\begin{lemma}
\label{lem:capsule-conflict-closure}
For a candidate capsule pair
\(\omega_{ij}^{\zeta}=(Q_i^{p:q},Q_j^{s:t})\) constructed from the mutual conflict
signature, all conflicts between \(Q_j^{s:t}\) and path \(p_i\) are contained in
\(Q_i^{p:q}\), i.e.,
\[
\operatorname{C}(a_i^k,a_j^\ell)=0,\quad
\forall \ell\in[s,t],\ k\notin[p,q].
\]
Therefore, after robot \(i\) exits \(Q_i^{p:q}\), its subsequent actions do not
conflict with \(Q_j^{s:t}\). The same property holds symmetrically for robot
\(j\).
\end{lemma}

\begin{proof}
For any \(\ell\in[s,t]\), the maximal constant conflict signature gives
\(\Gamma_{i\leftarrow j}(\ell)=\mathcal I\). By definition,
\(k\in\mathcal I\) if and only if
\(\operatorname{C}(a_i^k,a_j^\ell)=1\). Since \(Q_i^{p:q}\) is the minimum
contiguous span covering \(\mathcal I\), any \(k\notin[p,q]\) is not in
\(\mathcal I\), and thus cannot conflict with \(a_j^\ell\). Applying the same
argument to the opposite conflict signature proves the symmetric case.
\end{proof}

\begin{lemma}{\label{def:CPC}}
Given an action capsule pair \(\omega_{ij}^{\zeta}=(Q_i^{p:q},\) \(Q_j^{s:t})\), it is a CPC if at least one capsule is
non-singleton, i.e.,
\(q-p+1\ge 2\quad \text{or} \quad t-s+1\ge 2\), and the binary orientation \(z_\omega\in\{0,1\}\) satisfies the temporal constraints in Definition~\ref{def:AC} whenever switching.
% The initial orientation of \(\omega\) is inherited from the entry time. 
Specifically, if the two capsules satisfy
\(
\hat T_g(Q_i^{p:q})\le \hat T_s(Q_j^{s:t}),
\)
then the relation is \(Q_i^{p:q}\prec Q_j^{s:t}\),
and vice versa. The initial CPC orientation is consistent with the physical actions but is constructed at the capsule level.
% The initial orientation value \(z_\omega\) is assigned according to the temporal
% order induced by the initial HAPG.
\end{lemma}

% 证明CPC的条件
\begin{proof}
Without loss of generality, assume \(z_\omega=0\). Let \(n_i=q-p+1\) and \(n_j=t-s+1\). If \(n_i=n_j=1\), then both capsules are singletons. Since each capsule corresponds to one action, \(Q_i^{p:q}=\{a_i^p\}\) and \(Q_j^{s:t}=\{a_j^s\}\), switching to \(z_\omega=1\) is equivalent to reversing the two adjacent actions, which is an invalid switch. This corresponds to the following relation between two robots.
Now suppose that \(\max\{n_i,n_j\}\ge 2\). 
Initially, when \(z_\omega=0\), the relation is given by
\(T_g(Q_i^{p:q})\le T_s(Q_j^{s:t})\).
After switching to \(z_\omega=1\), Lemma~\ref{lem:capsule-conflict-closure} ensures that no subsequent action outside the paired capsules conflicts with the opposite capsule, and the absence of bidirectional dependencies within the CPC guarantees acyclicity between them. That is, before any action in
\(Q_i^{p:q}\) starts, all actions in \(Q_j^{s:t}\) that appear in
\(\Gamma_{i\leftarrow j}\) have already been completed, satisfying the constraint
\(T_g(Q_j^{s:t})\le T_s(Q_i^{p:q})\). Hence, \(\omega\) is a switchable CPC.
\end{proof}

As shown in Fig.~\ref{fig:cpc}, the paths of robots \(A\) and \(B\)
generate the following sequences of action capsule:
\(
\begin{aligned}
&\langle Q_A^{1:1}, Q_A^{2:2}, Q_A^{3:3}\rangle,
&\langle Q_B^{1:1},\ldots,Q_B^{5:5}\rangle.
\end{aligned}
\)
The conflict signature of robot \(B\) with respect to robot \(A\) over actions
\(1\) to \(4\) is given by
\(
\Gamma_{A\leftarrow B}(1,\ldots,4)=\{1,2\},
\)
which induces the CPC:\((Q_A^{1:2},Q_B^{1:4})\). Since
\(\hat T_g(Q_A^{1:2})\le \hat T_s(Q_B^{1:4})\), its initial orientation is
\(Q_A^{1:2}\prec Q_B^{1:4}\); hence, actions in \(Q_B^{1:4}\) are admitted only
after \(Q_A^{1:2}\) is completed. In this example, the Kiva-like robot can enter the
intersection only after the forklift has left it completely, while this capsule
order can still be switched online when both capsules remain $\mathrm{staged}$.

\subsubsection{CAPH} We continue to present the formal definition of CAPH and its online optimization process in this part.

\begin{definition}[Conjugate Action-Precedence Hypergraph]
The CAPH is defined as
\(
\mathcal H=(\mathcal Q,\Pi_f,\Omega_{\mathbf z}),
\)
where \(\mathcal Q\) is the set of action capsules,
\(\Pi_f\subseteq\mathcal Q\times\mathcal Q\) is the set of fixed
capsule precedences, \(\Omega\) is the set of CPCs, and \(\mathbf z=\{z_\omega\mid \omega\in\Omega\}\) is the set of binary variables to represent the orientations.
\end{definition}

For a given configuration \(\mathbf z\), CAPH induces an active
capsule precedence set
\(
\Pi(\mathbf z)=\Pi_f\cup\Omega(\mathbf z),
\)
where \(\Omega(\mathbf z)\) contains the dependencies selected
by each CPC. During online process, the orientation of a CPC can only be switched when both capsules are in the state of \(\mathrm{staged}\).

\begin{definition}[Inter-robot deadlock cycle in CAPH]
\label{def:deadlock-cycle}
Given the active precedence set \(\Pi(\mathbf z)\), a CAPH deadlock cycle exists if there is a sequence of capsules
\(
Q_{\alpha_1}^{b_1:e_1},Q_{\alpha_2}^{b_2:e_2},\ldots,Q_{\alpha_M}^{b_M:e_M}
\)
with \(M\geq2\), \(\alpha_m\in\mathcal R\), and at least two robots such that
\[
\bigl(Q_{\alpha_m}^{b_m:e_m},Q_{\alpha_{m+1}}^{b_{m+1}:e_{m+1}}\bigr)
\in\Pi(\mathbf z),\quad m=1,\ldots,M,
\]
where \(\alpha_{M+1}=\alpha_1\), \(b_{M+1}=b_1\), and \(e_{M+1}=e_1\). The action-level cycle is
\[
Z=\bigcup_{m=1}^{M}\{\,a_{\alpha_m}^{k}\mid k=b_m,\ldots,e_m\,\}.
\]
The set of all such action-level cycles within a horizon \(K\) is denoted by \(\mathfrak C_K\).
The CAPH is deadlock-free if \(\mathfrak C_K=\emptyset\).
\end{definition}

\subsubsection{Rolling Adjustment of CAPH}
\label{sec:opt-CAPH}
Inspired by the receding-horizon schedule re-ordering in SADG~\cite{SADG}, we
formulate the rolling adjustment as a MILP to optimize CAPH. Since paths are
released at most \(K\) steps in each scheduling cycle, the optimization only
involves the unreleased capsules within this bounded horizon and is therefore
lightweight. Let \(\mathcal Q_K\subseteq\mathcal Q\) be the set of
$\mathrm{staged}$ capsules, \(\Omega_K\subseteq\Omega_{\mathbf z}\) be
the CPCs, and \(\Pi_{f,K}\) be the fixed
precedences. Let \(z_\omega^0\) denote the current
orientation of CPC \(\omega\), and let \(\widehat\Delta(Q)\) be the remaining
execution duration of capsule \(Q\).

The MILP of rolling adjustment is formulated as:
\begin{equation}
\label{eq:caph-rolling-optimization}
\begin{aligned}
\min_{\substack{\mathbf z_K,\mathbf T_s,\mathbf T_g\\
z_\omega\in\{0,1\},\,\forall\omega\in\Omega_K}} \quad
& \sum_{Q\in\mathcal Q_K^{\mathrm{end}}} T_g(Q) \\
\mathrm{s.t.}\quad
& T_g(Q)\ge T_s(Q)+\widehat\Delta(Q),\\
& T_s(Q')\ge T_g(Q),\\
& T_s(Q')\ge T_g(Q)-M_{\mathrm B} z_\omega,\\
& T_s(Q)\ge T_g(Q')-M_{\mathrm B}(1-z_\omega).
\end{aligned}
\end{equation}
Here \(\mathcal Q_K^{\mathrm{end}}\) contains the last horizon capsule of each
robot considered in the current release, and \(M_{\mathrm B}\) is a sufficiently large
constant. The objective minimizes the cumulative predicted horizon-exit time.
The binary switch variables satisfy
\(\mathbf z_K=\{z_\omega\mid \omega\in\Omega_K\}\). $\widehat\Delta(Q)$ is the calculated remaining duration of capsule $Q$.
The duration constraint is imposed for all \(Q\in\mathcal Q_K\), the fixed
precedence constraint for all \((Q,Q')\in\Pi_{f,K}\), and the two
big-\(M_{\mathrm B}\) constraints for every CPC
\(\omega=(Q,Q')\in\Omega_K\). A switch with
\(z_\omega\ne z_\omega^0\) is admissible only when
\(\xi(Q)=\xi(Q')=\mathrm{staged}\), preventing dependencies from being reversed after either capsule has already been active. If a switched active precedence set contained a directed capsule cycle, the corresponding start and completion times would require at least one capsule to start after its own completion, contradicting the feasible temporal ordering. Hence the optimized hypergraph satisfies \(\mathfrak C_K(\mathbf z_K)=\emptyset\).

\begin{theorem}The CAPH-based persistent path execution is collision-free and deadlock-free with paths planned by the MRPP solver proposed in \ref{sec:path_planning}.
\end{theorem}
\begin{proof}
For a solvable instance, the MRPP solver returns a feasible ranked paths set for robots.
During planning, the SICD mechanism in Sec.~\ref{sec:SICD} enables MRPP to
exclude all roadmap-node occupations induced by swept footprints. The remaining
execution safety is then enforced by CAPH. Consider two inter-robot actions
\(a_{\alpha}^{k}\) and \(a_{\beta}^{m}\). If
\(\operatorname{C}(a_{\alpha}^{k},a_{\beta}^{m})=0\), they are non-conflicting
under the model in Sec.~\ref{sec:problem_formulation}. If
\(\operatorname{C}(a_{\alpha}^{k},a_{\beta}^{m})=1\), Def.~\ref{def:HAPG}
introduces an action-level precedence between them. CAPH lifts such conflict to fixed
capsule precedences or CPC, ensuring non-collision among robots under the precedence constraints.
By Lemma~\ref{def:CPC}, a valid CPC provides a capsule-level
temporal order for the conflicting capsules. Since a capsule is active only
after all its active predecessors have been \(\mathrm{completed}\), conflicting
actions are temporally separated during asynchronous execution. Therefore the
CAPH-based execution is collision-free.
For deadlock freedom, by Lemma~\ref{lem:hapg-safety-deadlock}, the HAPG
induced by the feasible MRPP paths is acyclic and admits deadlock-free
execution. As initial CAPH is obtained from this HAPG by lifting selected
action-level precedences to fixed capsule precedences or initial CPCs,  it admits the same feasible temporal ordering among capsules. During rolling adjustment in
Sec.~\ref{sec:opt-CAPH}, \(\mathrm{completed}\) or \(\mathrm{in\mbox{-}progress}\) capsule precedences cannot be reversed, and a CPC can be switched only when both
capsules are still \(\mathrm{staged}\). Each accepted switch also satisfies the
temporal constraints in \eqref{eq:caph-rolling-optimization}; otherwise a
directed capsule cycle would require some capsule to start after its own
completion. Thus \(\Pi(\mathbf z)\) remains acyclic after every accepted
adjustment. Hence the persistent path execution is deadlock-free.
\end{proof}

As illustrated in Fig.~\ref{fig:caph}, the example demonstrates a CAPH and its rolling adjustment process. During execution, $r_1$ and $r_2$ are delayed due to human-induced dynamic blockages or robot failures. After optimization, $r_3$ and $r_4$ are allowed to execute their paths earlier while collision-free operation is preserved.

\begin{figure}[htbp]
    \centering
    \begin{subfigure}[b]{\linewidth}
        \centering
        \includegraphics[width=\linewidth]{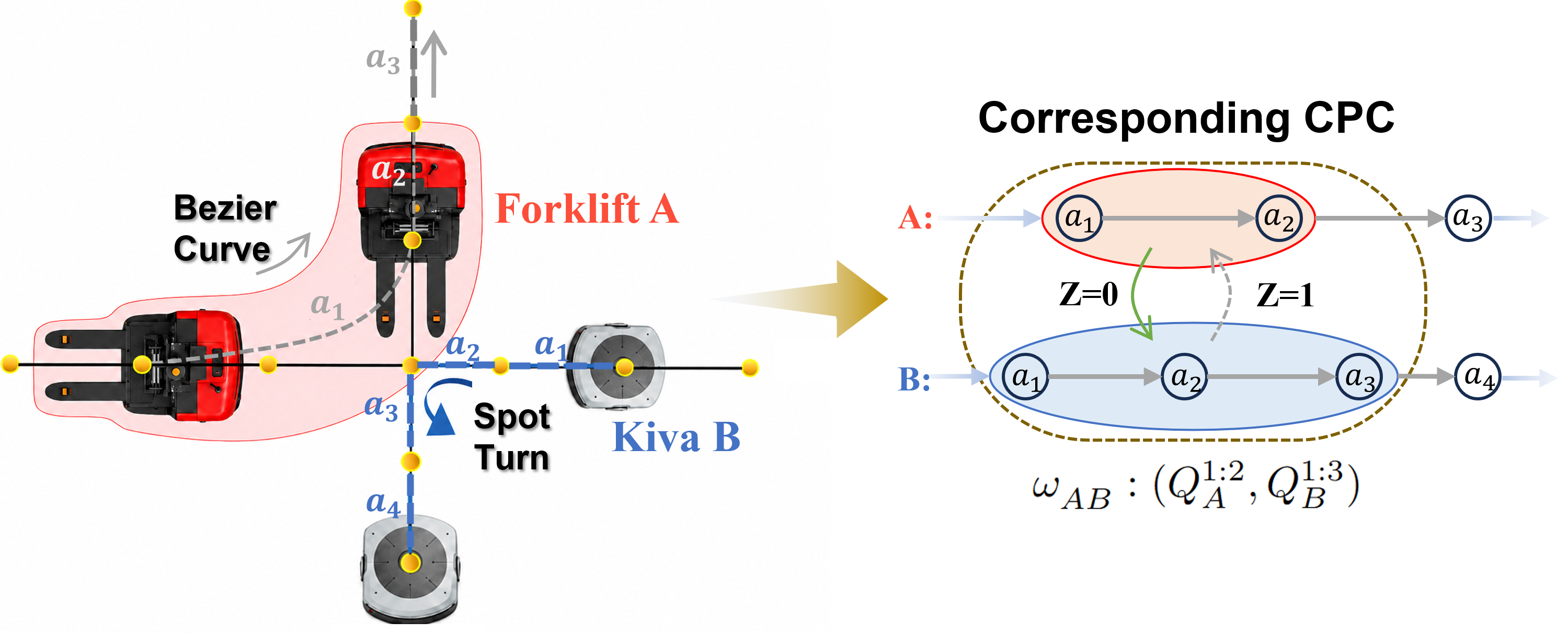}
        \caption{An example for CPC between a Forklift and a Kiva.}
        \label{fig:cpc}
    \end{subfigure}

    \vspace{-0.2em}

    \begin{subfigure}[b]{\linewidth}
        \centering
        \includegraphics[width=\linewidth]{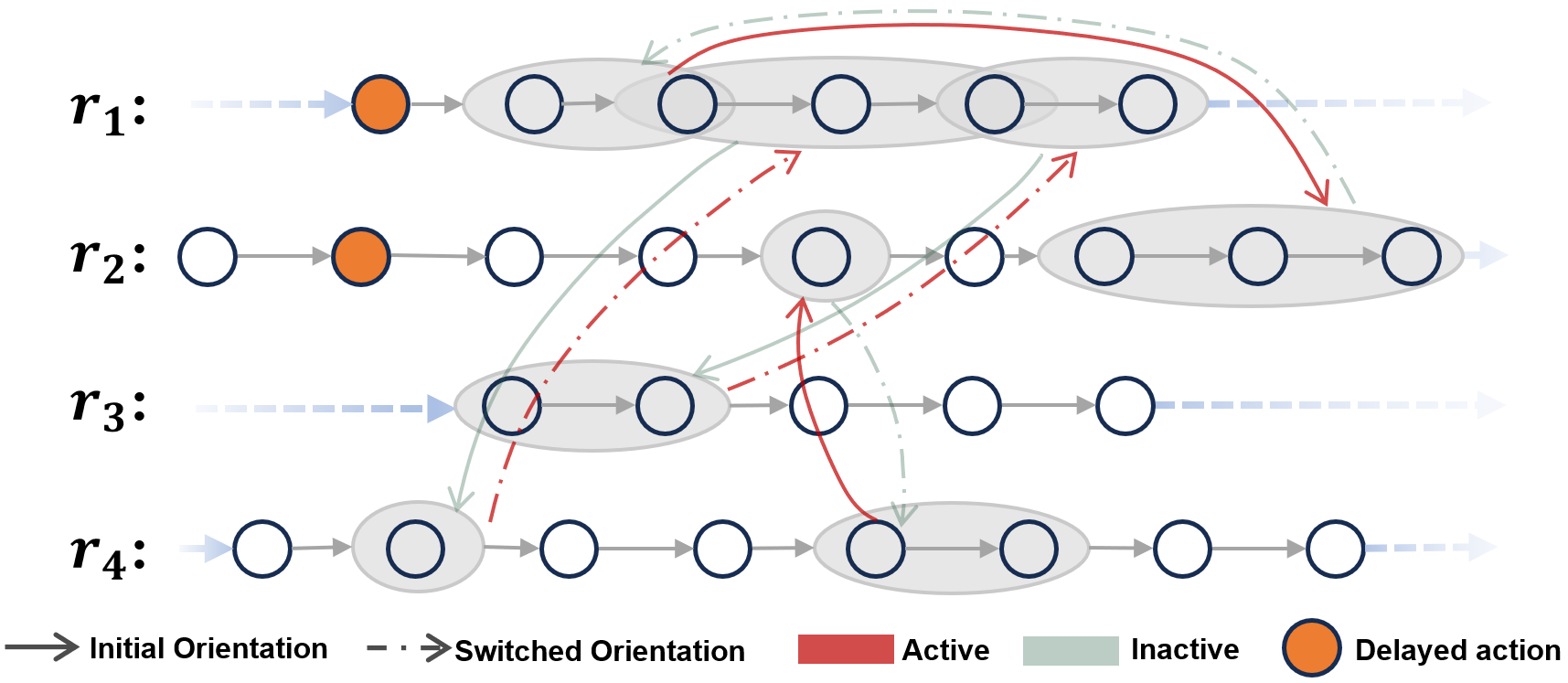}
        \caption{A CAPH for 4 robots, dependencies are switched for delayed events. The priority of action capsules of $r_3$ and $r_4$ are increased, while $r_1$ and $r_2$ have lagged behind.}
        \label{fig:caph}
    \end{subfigure}

    \caption{Illustration of examples of CAPH and its components.}
    \label{fig:cpc-caph}
\end{figure}

% CAPH is constructed in four steps. First, HAPG \(G_a=(A,E)\) is lifted to
% a singleton Action Capsule graph while preserving all original dependencies.
% Second, for each pair of robot paths, the Conflict Signature sequence is
% recorded along the compared path. Third, maximal consecutive singleton
% capsules with the same non-empty signature are grouped, and the conjugate
% capsule is obtained as the minimum contiguous span of that signature on the
% other path. Finally, candidate pairs satisfying the CPC condition are added
% to \(\Omega\), while non-switchable relations are retained in
% \(\Pi_f\), with all switch variables initialized according to the
% rank order \(\rho\).

% CAPH generalizes HAPG from pairwise action dependencies to capsule-level
% precedence relations. It preserves the original continuous-action execution
% model, reduces redundant inter-robot dependencies in local motion-induced
% conflict regions, and enables switchable passage orders over heterogeneous
% swept-region conflicts without replanning the underlying paths.

\section{Evaluation}

We evaluate the proposed method along three dimensions: real-time MRPP performance, robust path execution under external disturbances, and system-level lifelong coordination. For MRPP, the proposed SICD mechanism is integrated into representative solvers, including LaCAM*~\cite{lacam_star}, PIBT~\cite{pibt}, PBS~\cite{pbs}, ECBS~\cite{ECBS}, and CBS~\cite{CBS}.\footnote{We build the solvers based on the public implementations of LaCAM*~\url{https://github.com/Kei18/lacam0}, PIBT~\url{https://github.com/Kei18/pibt2}, and PBS/ECBS/CBS~\url{https://github.com/whoenig/libMultiRobotPlanning}.} These solvers are extended from grids to graph version. For each solver, we compare the original graph-agent implementation without continuous motion-collision checking, direct polygon-overlap (PO) detection, and the proposed SICD-based method.
For path execution, except for non-switchable HAPG, CAPH is compared with SG~\cite{sg}, an advanced segment-graph method that models robot motion at the segment level and greedily switches passage priorities. For system-level validation, SCALE is compared with two industrial baselines, CAEH~\cite{CAEH} and GNGW~\cite{glued-node}. CAEH alleviates congestion through a hierarchical topological architecture and local CBS-based conflict resolution, whereas GNGW uses glued nodes for safe operation and local give-way behaviors for collision avoidance.

To evaluate latency resilience under industrial disturbances, we introduce the uncertainty sources in Table~\ref{tab:execution-uncertainty} into the experiments in Secs.~\ref{subsec:execution-exp} and~\ref{subsec:system-level simulation}. Packet loss is sampled once per scheduling round; if triggered, the system loses the corresponding status update. Controller delay occurs independently before the start of each robot task, postponing execution by 5 s. Human-factor disturbances are sampled for each robot-task pair with a probability of 2\%; once triggered, the affected robot remains stationary for 120 s before resuming operation.

All experiments involve three robot types: forklifts, mobile manipulators, and Kiva robots. Table~\ref{tab:robot-parameters} summarizes the robot parameters, with the three robot types maintained at a fixed number ratio of $4:1:5$.

All programs are written in C++ and run on a workstation with an Intel Core i9-14900K CPU and 32 GB RAM.

\begin{table}[htbp]
  \centering
  \caption{Settings for External Disturbances in Experiments}
  \setlength{\tabcolsep}{4pt}
  \resizebox{0.65\columnwidth}{!}{%
  \begin{tabular}{@{}l c l@{}}
    \toprule
    \textbf{Item} & \textbf{Probability} & \textbf{Pattern} \\
    \midrule
    Packet loss & $1\%$ / planning cycle & State update loss \\
    Controller delay & $0.5\%$ / robot start & $5$ s start delay \\
    Human factor & $2\%$ / robot-task pair & 120 s pause \\
    \bottomrule
  \end{tabular}
  }
  \label{tab:execution-uncertainty}
\end{table}

\begin{table}[htbp]
  \centering
  \caption{Parameters of Heterogeneous Robots}
  \setlength{\tabcolsep}{4pt}
  \resizebox{\columnwidth}{!}{%
  \begin{tabular}{@{}l c c c@{}}
    \toprule
    \textbf{Parameter} & \textbf{Forklift (FL.)} & \textbf{Mobile Manipulator (M.M.)} & \textbf{Kiva} \\
    \midrule
    Unloaded robot size $(m \times m)$ & $(2.10 \times 0.96)$ & $(0.85 \times 0.85)$ & $(0.76 \times 0.76)$ \\
    Loaded robot size $(m \times m)$ & $(2.10 \times 1.20)$ & $(0.85 \times 0.90)$ & $(0.76 \times 0.76)$ \\
    Max velocity $(m/s)$ & $1.6$ & $1.0$ & $1.3$ \\
    Max acceleration $(m/s^2)$ & $0.6$ & $0.5$ & $0.7$ \\
    Max braking acceleration $(m/s^2)$ & $0.8$ & $0.7$ & $0.9$ \\
    Min turning radius $(m)$ & $1.6$ & $0.8$ & $0$ \\
    Safety radius $\epsilon$ $(m)$ & $0.3$ & $0.2$ & $0.15$ \\
    Battery consumption rate $(\%/h)$ & $15\%/h$ & $17\%/h$ & $8\%/h$\\
    Charging rate $(\%/h)$ & $50\%/h$ & $55\%/h$ & $80\%/h$\\
    \bottomrule
  \end{tabular}
  }
  \label{tab:robot-parameters}
\end{table}

% GNGW(glued-node give-away)
% CAEH(Corridor-Aware Extended-Horizon)

\subsection{Performance of One-shot MRPP}

We first evaluate conflict detection on MRPP instances sampled from the long-term operation in Sec.~\ref{subsec:system-level simulation}. Each triggered local MRPP group is recorded, yielding 10,000 instances with 2--20 robots. Fig.~\ref{fig:planning1} reports the cumulative solving-time distribution under a 500 ms timeout. LaCAM* with SICD solves 90.0\% of the instances while explicitly considering continuous motion-conflict detection (CCD), only 9.7\% lower than the original graph-agent solver without CCD. In contrast, the PO-based variant solves only 23.2\%. Fig.~\ref{fig:planning_overview} further reports the average total detections (TD) and true-positive rate (TPR). Introducing CCD increases the number of checks because swept-volume interactions reveal physically relevant conflicts beyond the assumption of agents. Overall, SICD provides the required physical safety screening while preserving much of the original MRPP search efficiency.

To further verify the applicability of SICD, we evaluate it on three maps from the public \textit{MovingAI MAPF} benchmark\footnote{\url{https://movingai.com/benchmarks/mapf}}. In these maps, each edge is assigned a length of 1.5 m, and edge curvature is adjusted according to robot turning-radius constraints. For each map, we randomly generate 50 instances under 11 robot-density settings from 0.02 to 0.30, with a 30 s solving timeout. Fig.~\ref{fig:success_rate} shows that SICD-based solvers maintain high success rates across different structures of maps. Compared with the original solvers without CCD, the success-rate reduction is generally below 8\% when the robot density is below 0.20 and remains within 22\% in denser settings. In particular, it largely preserves the original solver efficiency when the number of robots is below 100. By contrast, PO substantially reduces scalability: even advanced solvers such as LaCAM* and PIBT lose the ability to solve in time at low densities. For example, on den312d, CBS with PO fails to solve any instance even at density 0.02. Fig.~\ref{fig:solving_time} reports the median solving time over solvable cases, with shaded bands indicating the minimum--maximum range. PO-based solvers spend substantial effort on repeated geometric overlap tests, whereas SICD keeps the overhead close to the original search and remains practical for MRPP.

\begin{figure}[t]
    \centering
    \captionsetup[subfigure]{font=scriptsize,skip=1pt}
    \begin{subfigure}[t]{0.5\linewidth}
        \centering
        \vspace{0pt}
        \includegraphics[width=\linewidth]{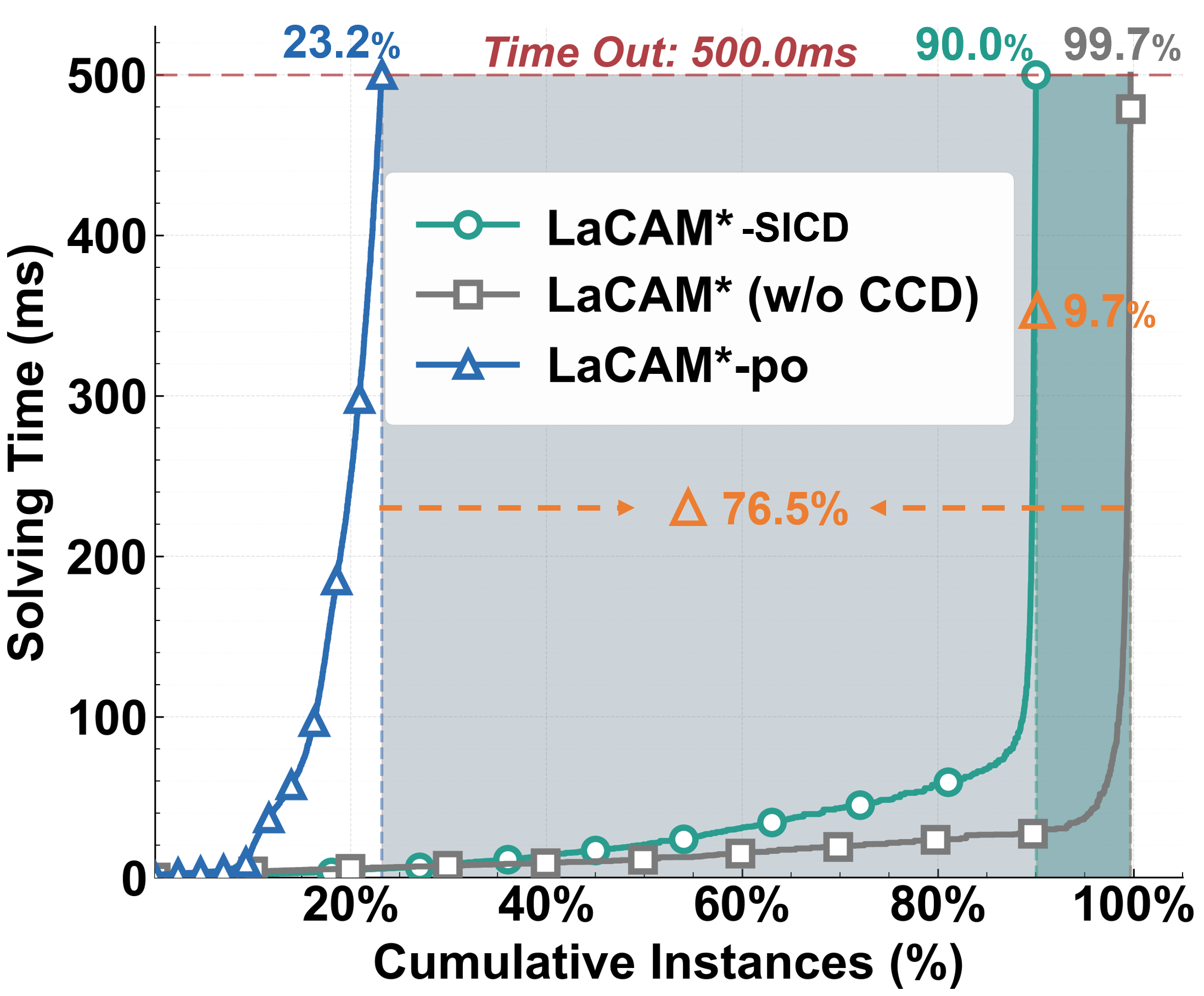}
        \caption{Performance of LaCAM* variants.}
        \label{fig:planning1}
    \end{subfigure}
    \hfill
    \begin{subfigure}[t]{0.43\linewidth}
        \centering
        \vspace{0pt}
        \begingroup
        \scriptsize
        \setlength{\tabcolsep}{4pt}
        \renewcommand{\arraystretch}{0.94}
        \resizebox{0.93\linewidth}{!}{%
        \begin{tabular}{@{}c c c c c@{}}
        \toprule
        \multirow{2}{*}{\textbf{$N_r$}}
        & \multicolumn{2}{c}{\textbf{w/o CCD}}
        & \multicolumn{2}{c}{\textbf{SICD}} \\
        \cmidrule(lr){2-3} \cmidrule(lr){4-5}
        & \textbf{TD} & \textbf{TPR(\%)} & \textbf{TD} & \textbf{TPR(\%)} \\
        \midrule
        2  & 26  & 1.3 & 344  & 8.7  \\
        4  & 45  & 3.2 & 798  & 5.8  \\
        6  & 73  & 3.3 & 853  & 15.8 \\
        8  & 95  & 3.7 & 1913 & 11.5 \\
        10 & 122 & 5.7 & 1813 & 9.9  \\
        12 & 144 & 4.9 & 1600 & 7.5  \\
        14 & 205 & 7.6 & 2109 & 8.0  \\
        16 & 297 & 9.2 & 2661 & 13.5 \\
        18 & 274 & 8.2 & 3054 & 11.3 \\
        20 & 277 & 7.5 & 5031 & 12.1 \\
        \bottomrule
        \end{tabular}
        }

        \vspace{0.6ex}
        \parbox{0.93\linewidth}{\centering\scriptsize TD: total detections; TPR: true positive rate.}
        \endgroup
        \vspace{0pt}
        \caption{Number of conflicts during search.}
        \label{tab:exp3_conflict_rate}
    \end{subfigure}
    \caption{Planning performance on an industrial roadmap graph.}
    \label{fig:planning_overview}
\end{figure}

\begin{figure*}[htbp]
    \centering
    \includegraphics[width=0.9\linewidth]{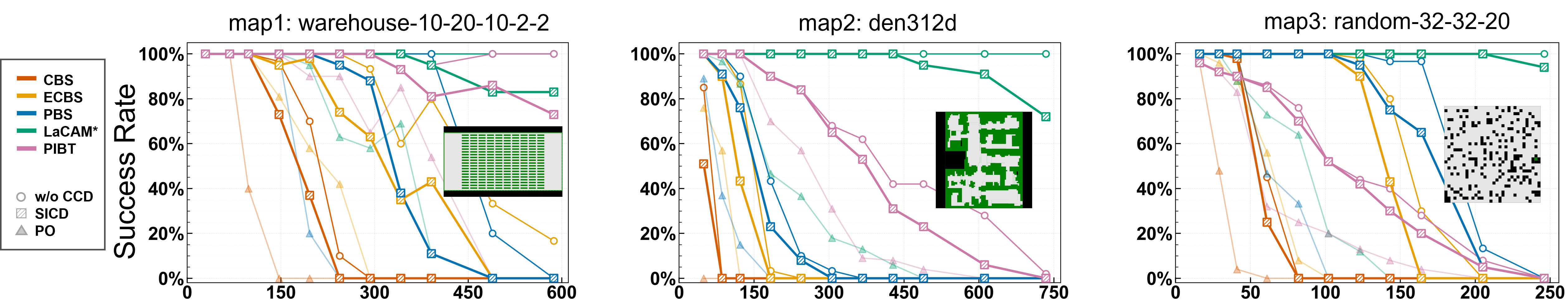}
    \caption{Success rate versus the number of robots for representative MRPP solvers with and without CCD.
}
    \label{fig:success_rate}
\end{figure*}

\begin{figure*}[htbp]
    \centering
    \includegraphics[width=0.93\linewidth]{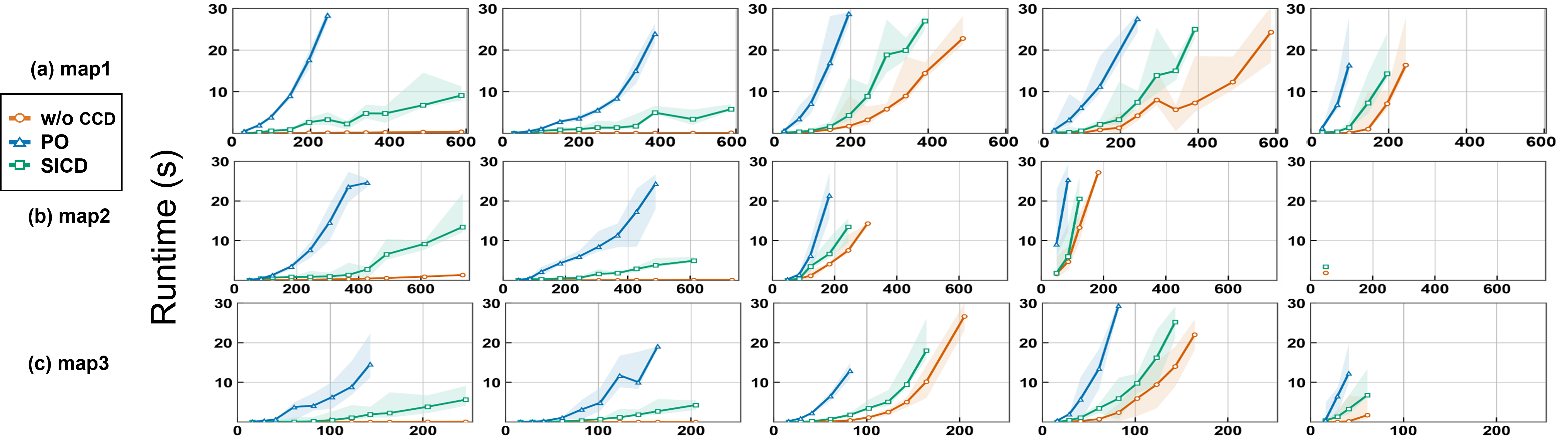}
    \caption{Computation time of successful MRPP instances in Fig.~\ref{fig:success_rate}.}
    \label{fig:solving_time}
\end{figure*}

\subsection{Performance of Path Execution under Uncertainty}
\label{subsec:execution-exp}

We evaluate path execution policies under external disturbances. Statistics from the system-level simulation in Sec.~\ref{subsec:system-level simulation} show that local robot groups in SCALE rarely contain more than 20 robots. To further test scalability, we randomly generate solvable cases with group sizes
\(N_r\in\{10,20,30,40,50\}\) in the same scenario and sample 50 cases for each size. The disturbance settings follow Table~\ref{tab:execution-uncertainty}.
Table~\ref{tab:execution_comparison} shows that CAPH achieves the shortest execution time for all group sizes. Compared with the non-switchable HAPG, CAPH reduces execution time by up to 18.4\% when the scale is 50. Since HAPG strictly follows the precedence relations generated by MRPP, a local delay can propagate through fixed dependencies and increase the completion time of the entire group. SG also allows dependency switching, but its execution-time reduction over HAPG remains below 8.0\% in all tested cases, indicating that greedy switching over coarse action segments provides limited resilience to external disturbances.

The advantage of CAPH over SG becomes more pronounced as the group size increases. In the 50-robot cases, it reduces the average execution time by about 110s compared with SG. It also achieves this improvement with far fewer dependency switches: its switch count is about one third of SG across all cases. The capsule-level dependency model therefore avoids many ineffective or negative switches globally, reduces redundant switching computation, and helps prevent livelock caused by dependency oscillations.

\begin{table}[htbp]
  \centering
  \caption{Comparison of different execution strategies}
  \setlength{\tabcolsep}{3pt}
  \resizebox{0.9\linewidth}{!}{%
  \begin{tabular}{@{}c c c c c c@{}}
    \toprule
    \multirow{2}{*}{\textbf{$N_r$}} 
    & \multicolumn{3}{c}{\textbf{Total Time of Execution / s} ($\downarrow$)} 
    & \multicolumn{2}{c}{\textbf{Number of Switches}} \\
    \cmidrule(lr){2-4} \cmidrule(lr){5-6}
    & \textbf{HAPG} & \textbf{SG} & \textbf{CAPH} 
    & \textbf{SG} & \textbf{CAPH} \\
    \midrule
    10 & \stat{239.6}{16.9}  & \stat{237.6}{13.1}  & \statbf{212.3}{9.7}   & 33.5  & 10.7 \\
    20 & \stat{395.9}{68.5}  & \stat{387.4}{74.2}  & \statbf{297.1}{38.1}  & 116.4 & 33.5 \\
    30 & \stat{519.0}{90.1}    & \stat{502.1}{92.2}  & \statbf{427.5}{76.7}  & 181.3 & 57.1 \\
    40 & \stat{621.4}{130.5} & \stat{581.3}{105.3} & \statbf{533.8}{91.0}  & 240.3 & 106.3 \\
    50 & \stat{793.3}{199.9} & \stat{756.1}{168.5} & \statbf{647.8}{112.5} & 292.4 & 126.8 \\
    \bottomrule
  \end{tabular}
  }
  \vspace{0.2mm}
  \parbox{0.9\linewidth}{\footnotesize Values in parentheses denote standard deviations.}
  \label{tab:execution_comparison}
\end{table}

\subsection{System-level Simulation} 
\label{subsec:system-level simulation}
We conduct a lifelong system-level simulation in our real-world industrial project with a fleet of 300 heterogeneous robots. The roadmap graph contains 24,356 nodes, 60109 edges, and 1500 fixed storage locations.
Under the external-disturbance settings in
Table~\ref{tab:execution-uncertainty} and the robot parameters in
Table~\ref{tab:robot-parameters}, the system randomly releases 80 tasks every
60s. Each task includes pickup and drop-off operations, with a 5s operation
time. The experiment runs for three days.

The comparison of throughput over time is shown in
Fig.~\ref{fig:throughput-comparison}. SCALE rapidly increases its throughput
during the startup phase and reaches about \(50\) tasks/min within the first
10h, while CAEH and GNGW remain below about \(25\) tasks/min in the same
period. During the main operating interval from 10h to 60h, SCALE
maintains a high throughput level, mostly between \(40\) and \(55\) tasks/min
and peaking at about \(54\) tasks/min. In contrast, CAEH stays
around \(30\)--\(35\) tasks/min and requires multiple local interventions when
unsolved local conflicts occur, as marked by the crosses in the figure. Its
average throughput is therefore about 30\% lower than that of SCALE. GNGW
achieves a short-term peak of about \(35\) tasks/min, but its throughput starts
to collapse after about 30h and becomes nearly zero around 42h because local
give-way behaviors cannot recover from accumulated deadlocks. Consequently,
GNGW finishes less than 40\% of the total tasks, whereas SCALE completes the
released tasks and returns all robots to the idle state at the end.
Fig.~\ref{fig:response-time-comparison} reports the distribution of task
response time, measured as the waiting time between task assignment and the
moment when the assigned robot starts moving. SCALE has the most concentrated
distribution, with its main peak within about 10--20s and a fast decay after
60s. CAEH and GNGW exhibit right-shifted distributions, with peaks around
30s and 40s, respectively, and visibly heavier long tails caused by postponed
task admission under congestion and unresolved local conflicts. The average
task response times of SCALE, CAEH, and GNGW are 28s, 51s, and 69s,
respectively, meaning that SCALE reduces the average response time by about
45\% compared with CAEH and about 59\% compared with GNGW. Overall, the
proposed method achieves higher sustained throughput and shorter task response
time in the lifelong industrial simulation.

\begin{figure}[hpb]
    \centering
    \begin{subfigure}[t]{\linewidth}
        \centering
        \includegraphics[width=0.95\linewidth]{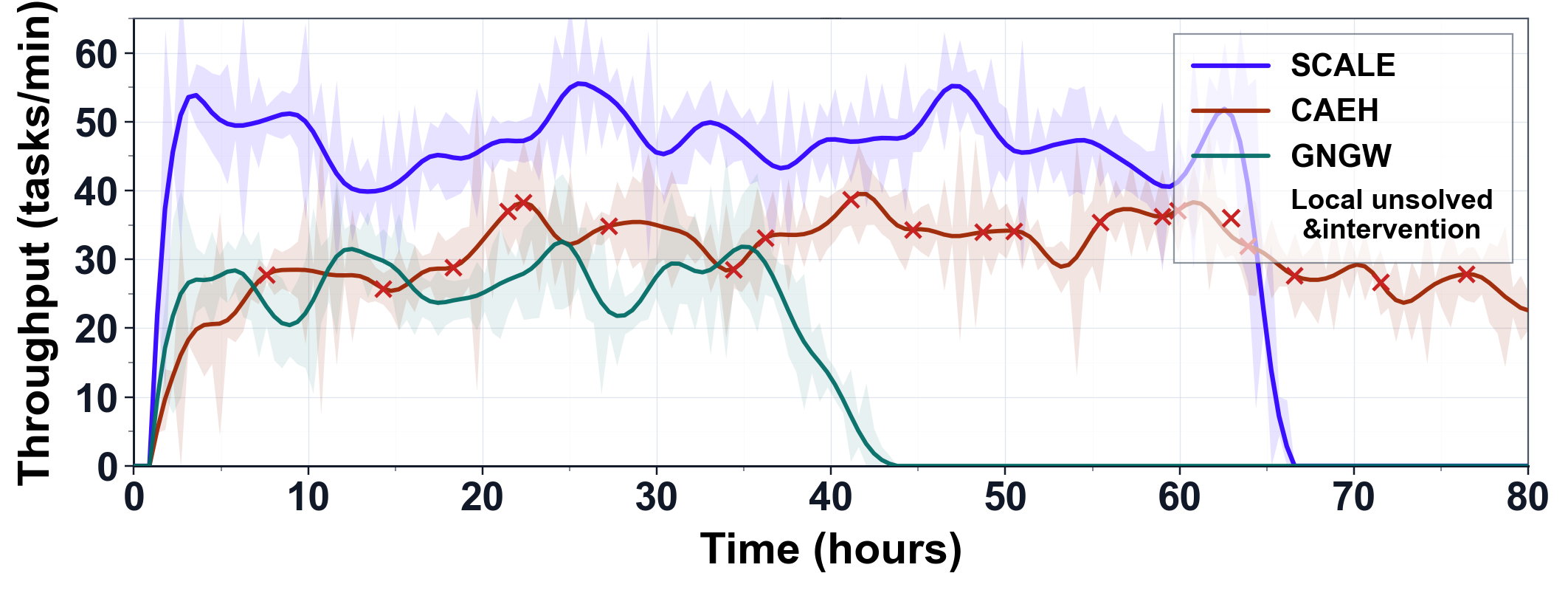}
        \caption{Comparison of throughput per minute.}
        \label{fig:throughput-comparison}
        \end{subfigure}

    \vspace{-0.5mm}
    \begin{subfigure}[t]{\linewidth}
        \centering
        \includegraphics[width=0.8\linewidth]{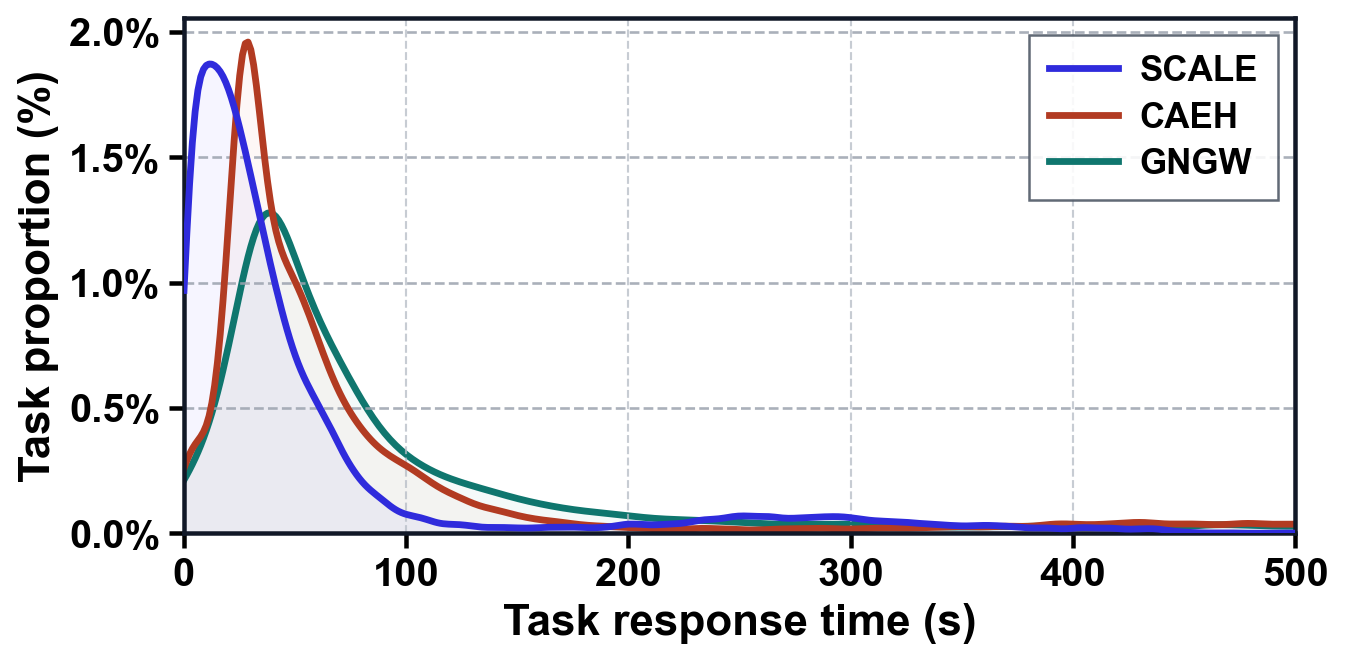}
        \caption{Comparison of response time for tasks.}
        \label{fig:response-time-comparison}
    \end{subfigure}
    \caption{Comparison of lifelong performance for running 300 robots for three days under three scheduling frameworks.}
    \label{fig:simulation-performance}
\end{figure}

\subsection{Industrial Deployment}

\begin{figure}[htpb]
    \centering
    \includegraphics[width=\linewidth]{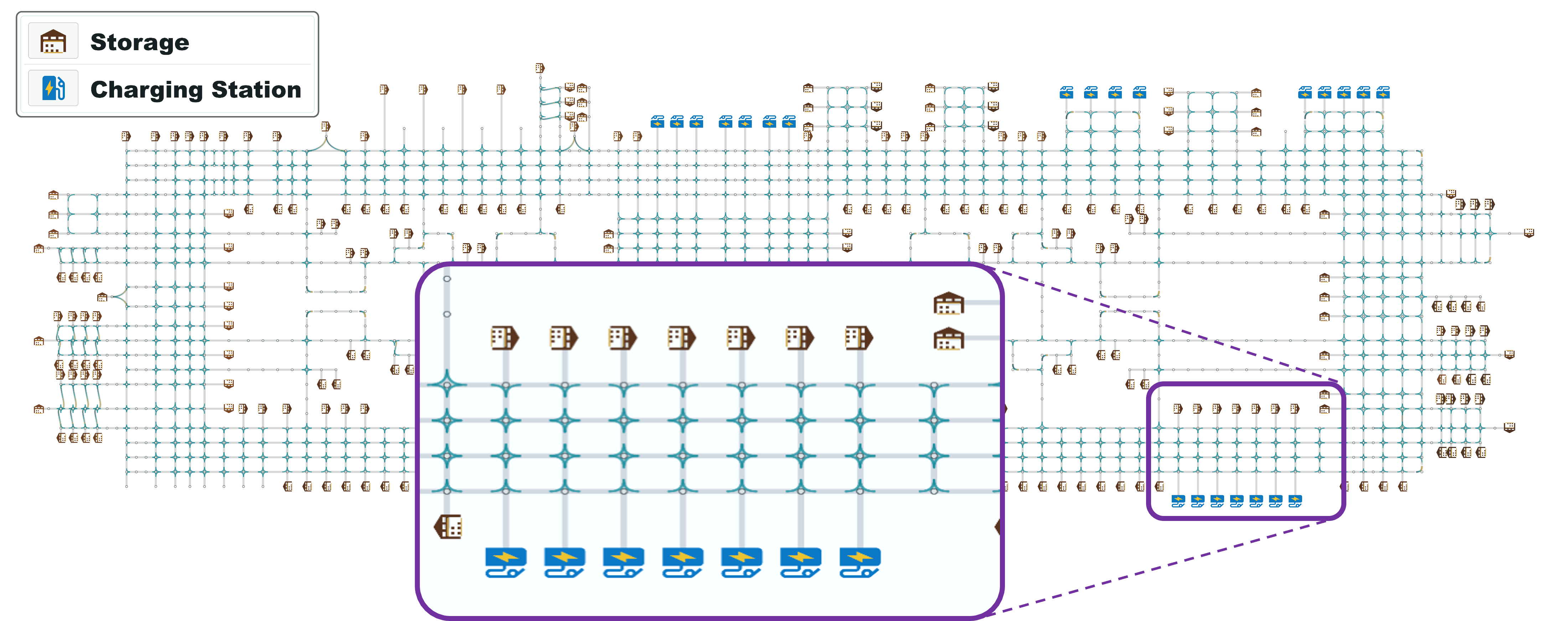}
    \caption{The roadmap network in a warehousing factory.}
    \label{fig:roadmap}
\end{figure}

The deployment validation is conducted in a real-world industrial warehousing scenario with a heterogeneous fleet of 100 robots; their physical parameters are summarized in Table~\ref{tab:robot-parameters}. The real-world deployment scene is shown in Fig.~\ref{fig:scen}, and the corresponding roadmap graph is shown in Fig.~\ref{fig:roadmap}. The site contains multiple storage blocks, long trunk corridors, dense aisle intersections, and several charging areas, which together create frequent merging, crossing, and temporary blocking situations during normal operation. The warehouse roadmap graph contains 5,839 nodes, 19,670 edges, and 274 fixed storage locations.
\begin{figure}[htpb]
    \centering
    \includegraphics[width=0.86\linewidth]{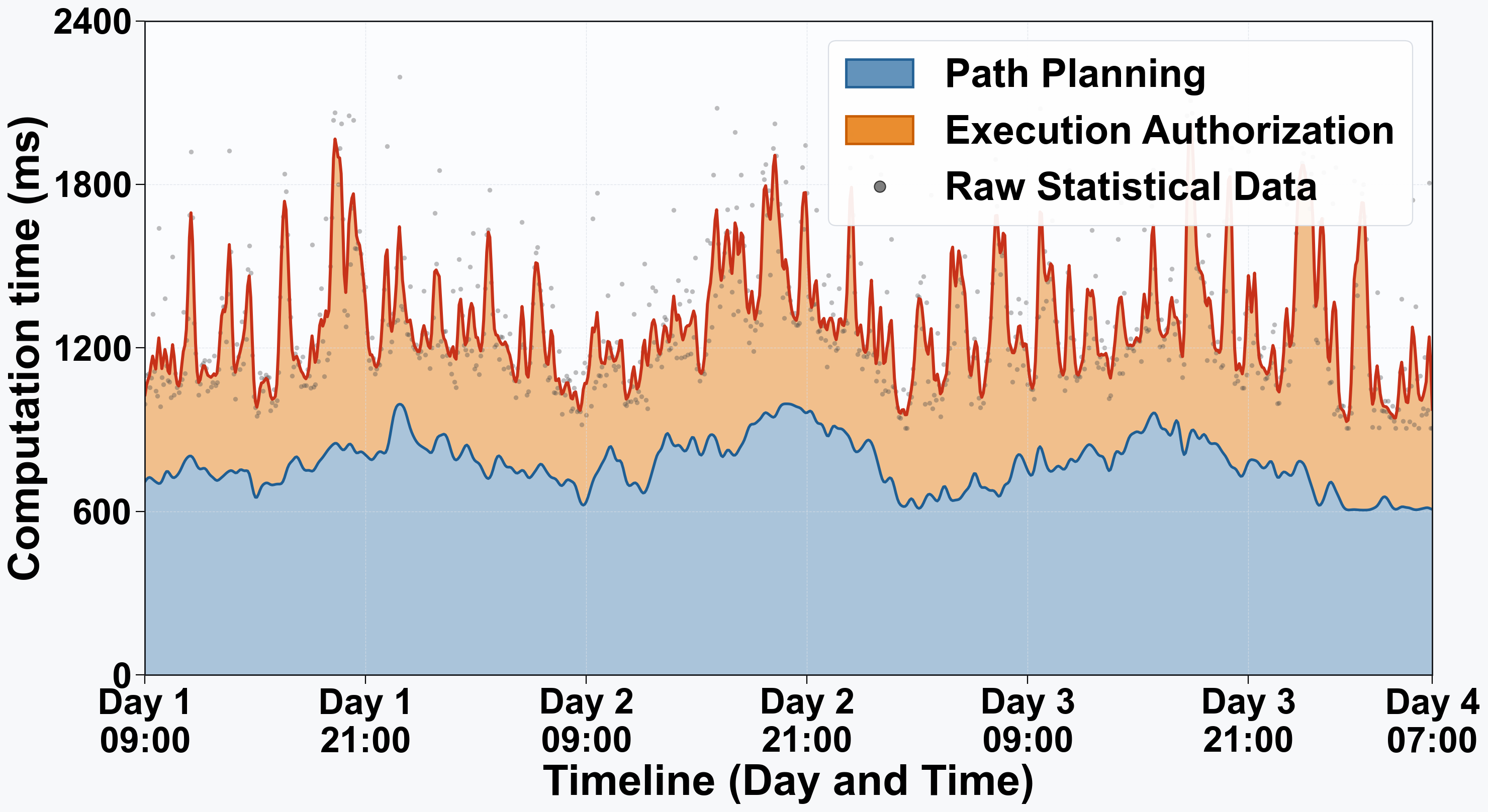}
    \caption{Computational time distribution in deployment.}
    \label{fig:warehouse-compute-time}
\end{figure}
During the monitored multi-day operation, the 100 robots complete approximately 30,000 tasks in total, corresponding to more than 100 tasks per robot per day on average. Scheduling is triggered every 10s, and each round processes robot states, updates unfinished path prefixes, resolves newly detected local conflicts, and authorizes executable paths. Fig.~\ref{fig:warehouse-compute-time} shows the edge-side computation-time distribution during operation. The path planning time in each scheduling round remains stable within 700--1800 ms. Path authorization relies on rolling optimization and therefore incurs additional computational overhead, particularly when robot priorities are need to be adjusted in response to external disturbances. Even so, the total edge-side computation remains well below the period between scheduling loops, leaving sufficient time for device-side execution.

\section{Conclusion}

This paper studies scalable coordination for heterogeneous robot fleet in industrial environments, where system efficiency depends not only on collision-free path planning under realistic constraints, but also on reliable, robust asynchronous path execution under external disturbances. The proposed framework treats planning and execution as a coupled online process. It commits bounded paths to enable controllable scheduling through online conflict resolution, and regulates execution relations via an action-precedence dependency hypergraph, on which the execution order of robots can be receding-horizon optimized in response to external disturbances. In this sense, the framework bridges graph-based multi-robot planning and the practical requirements of scheduling dense industrial robot fleet.
Extensive simulations and real-world factory deployment demonstrate the feasibility and efficiency of the proposed method for large-scale heterogeneous robot fleet. Our future work will focus on extending the pipeline from roadmap graph-based coordination to more general multi-robot trajectory planning, so that continuous-time motion planning and execution can be integrated within a more unified representation, thereby broadening the applicability of multi-robot systems.

\bibliographystyle{IEEEtran}
\bibliography{IEEEabrv,Bibliography}

\begin{thebibliography}{10}
\providecommand{\url}[1]{#1}
\csname url@samestyle\endcsname
\providecommand{\newblock}{\relax}
\providecommand{\bibinfo}[2]{#2}
\providecommand{\BIBentrySTDinterwordspacing}{\spaceskip=0pt\relax}
\providecommand{\BIBentryALTinterwordstretchfactor}{4}
\providecommand{\BIBentryALTinterwordspacing}{\spaceskip=\fontdimen2\font plus
\BIBentryALTinterwordstretchfactor\fontdimen3\font minus \fontdimen4\font\relax}
\providecommand{\BIBforeignlanguage}[2]{{%
\expandafter\ifx\csname l@#1\endcsname\relax
\typeout{** WARNING: IEEEtran.bst: No hyphenation pattern has been}%
\typeout{** loaded for the language `#1'. Using the pattern for}%
\typeout{** the default language instead.}%
\else
\language=\csname l@#1\endcsname
\fi
#2}}
\providecommand{\BIBdecl}{\relax}
\BIBdecl

\bibitem{zone-control}
S.~Reveliotis, ``An mpc scheme for traffic coordination in open and irreversible, zone-controlled, guidepath-based transport systems,'' \emph{IEEE Transactions on Automation Science and Engineering}, vol.~17, no.~3, pp. 1528--1542, 2020.

\bibitem{petri-net}
J.~Luo, Y.~Wan, W.~Wu, and Z.~Li, ``Optimal petri-net controller for avoiding collisions in a class of automated guided vehicle systems,'' \emph{IEEE Transactions on Intelligent Transportation Systems}, vol.~21, no.~11, pp. 4526--4537, 2019.

\bibitem{glued-node}
Z.~Xing, H.~Yue, T.~Zhang, W.~Wu, and R.~Hu, ``Collision avoidance and give way of heterogeneous and variable-sized multiple mobile robots based on glued nodes,'' \emph{IEEE Transactions on Automation Science and Engineering}, vol.~21, no.~3, pp. 3627--3638, 2023.

\bibitem{mapf-lns2}
J.~Li, Z.~Chen, D.~Harabor, P.~J. Stuckey, and S.~Koenig, ``Mapf-lns2: Fast repairing for multi-agent path finding via large neighborhood search,'' in \emph{Proceedings of the AAAI Conference on Artificial Intelligence}, vol.~36, no.~9, 2022, pp. 10\,256--10\,265.

\bibitem{lacam}
K.~Okumura, ``{LaCAM}: Search-based algorithm for quick multi-agent pathfinding,'' in \emph{Proceedings of the AAAI Conference on Artificial Intelligence}, vol.~37, no.~10, 2023, pp. 11\,655--11\,662.

\bibitem{lacam_star}
{K. Okumura}, ``Improving {LaCAM} for scalable eventually optimal multi-agent pathfinding,'' in \emph{Proceedings of the Thirty-Second International Joint Conference on Artificial Intelligence, {IJCAI-23}}, 8 2023, pp. 243--251.

\bibitem{large-agent-li}
J.~Li, P.~Surynek, A.~Felner, H.~Ma, T.~S. Kumar, and S.~Koenig, ``Multi-agent path finding for large agents,'' in \emph{Proceedings of the AAAI Conference on Artificial Intelligence}, vol.~33, no.~01, 2019, pp. 7627--7634.

\bibitem{CCBS-1}
E.~Boyarski, R.~Stern, K.~S. Yakovlev, and A.~Andreychuk, ``Improving {Continuous}-time {Conflict} {Based} {Search}.'' \emph{arXiv: Artificial Intelligence}, Jan. 2021.

\bibitem{k-robut}
Z.~Chen, D.~D. Harabor, J.~Li, and P.~J. Stuckey, ``Symmetry breaking for k-robust multi-agent path finding,'' in \emph{Proceedings of the AAAI Conference on Artificial Intelligence}, vol.~35, no.~14, 2021, pp. 12\,267--12\,274.

\bibitem{ADG}
W.~H{\"o}nig, S.~Kiesel, A.~Tinka, J.~W. Durham, and N.~Ayanian, ``Persistent and robust execution of mapf schedules in warehouses,'' \emph{IEEE Robotics and Automation Letters}, vol.~4, no.~2, pp. 1125--1131, 2019.

\bibitem{SADG}
A.~Berndt, N.~Van~Duijkeren, L.~Palmieri, A.~Kleiner, and T.~Keviczky, ``Receding horizon re-ordering of multi-agent execution schedules,'' \emph{IEEE Transactions on Robotics}, vol.~40, pp. 1356--1372, 2023.

\bibitem{pure-rule-agvs}
Y.~Sun, N.~Zhao, L.~Tang, and L.~Luo, ``Breaking the limit on the number of robots through the conflict-free scheduling in robotic mobile fulfillment systems,'' \emph{IEEE Transactions on Automation Science and Engineering}, vol.~22, pp. 7324--7334, 2024.

\bibitem{CAEH}
A.~Bonetti, S.~Proia, S.~Guidetti, and L.~Sabattini, ``Agv traffic management in automated industrial plants: An enhanced lifelong multi-agent path finding approach,'' in \emph{2024 IEEE 20th International Conference on Automation Science and Engineering (CASE)}.\hskip 1em plus 0.5em minus 0.4em\relax IEEE, 2024, pp. 626--632.

\bibitem{well-formed}
H.~Ma, W.~H{\"o}nig, T.~S. Kumar, N.~Ayanian, and S.~Koenig, ``Lifelong path planning with kinematic constraints for multi-agent pickup and delivery,'' in \emph{Proceedings of the AAAI Conference on Artificial Intelligence}, vol.~33, no.~01, 2019, pp. 7651--7658.

\bibitem{RHCR}
J.~Li, A.~Tinka, S.~Kiesel, J.~W. Durham, T.~S. Kumar, and S.~Koenig, ``Lifelong multi-agent path finding in large-scale warehouses,'' in \emph{Proceedings of the AAAI Conference on Artificial Intelligence}, vol.~35, no.~13, 2021, pp. 11\,272--11\,281.

\bibitem{ECBS}
M.~Barer, G.~Sharon, R.~Stern, and A.~Felner, ``Suboptimal variants of the conflict-based search algorithm for the multi-agent pathfinding problem,'' in \emph{Proceedings of the Seventh Annual Symposium on Combinatorial Search}, 2014, pp. 19--27.

\bibitem{pbs}
H.~Ma, D.~Harabor, P.~J. Stuckey, J.~Li, and S.~Koenig, ``Searching with consistent prioritization for multi-agent path finding,'' in \emph{Proceedings of the AAAI conference on artificial intelligence}, vol.~33, no.~01, 2019, pp. 7643--7650.

\bibitem{CBS}
G.~Sharon, R.~Stern, A.~Felner, and N.~R. Sturtevant, ``Conflict-based search for optimal multi-agent pathfinding,'' \emph{Artificial intelligence}, vol. 219, pp. 40--66, 2015.

\bibitem{pibt}
K.~Okumura, M.~Machida, X.~D{\'e}fago, and Y.~Tamura, ``Priority inheritance with backtracking for iterative multi-agent path finding,'' \emph{Artificial Intelligence}, vol. 310, p. 103752, 2022.

\bibitem{hetero-mapf}
D.~Atzmon, Y.~Zax, E.~Kivity, L.~Avitan, J.~Morag, and A.~Felner, ``Generalizing multi-agent path finding for heterogeneous agents,'' in \emph{Proceedings of the International Symposium on Combinatorial Search}, vol.~11, no.~1, 2020, pp. 101--105.

\bibitem{2022CBS-dynamics}
J.~Kottinger, S.~Almagor, and M.~Lahijanian, ``Conflict-based search for multi-robot motion planning with kinodynamic constraints,'' in \emph{2022 IEEE/RSJ International Conference on Intelligent Robots and Systems (IROS)}.\hskip 1em plus 0.5em minus 0.4em\relax IEEE, 2022, pp. 13\,494--13\,499.

\bibitem{db-lacam}
A.~Moldagalieva, K.~Okumura, A.~Prorok, and W.~H{\"o}nig, ``db-lacam: Fast and scalable multi-robot kinodynamic motion planning with discontinuity-bounded search and lightweight mapf,'' \emph{arXiv preprint arXiv:2512.06796}, 2025.

\bibitem{p-robut}
D.~Atzmon, R.~Stern, A.~Felner, N.~R. Sturtevant, and S.~Koenig, ``Probabilistic robust multi-agent path finding,'' in \emph{Proceedings of the International Conference on Automated Planning and Scheduling}, vol.~30, 2020, pp. 29--37.

\bibitem{STN-ma}
W.~H{\"o}nig, T.~Kumar, L.~Cohen, H.~Ma, H.~Xu, N.~Ayanian, and S.~Koenig, ``Multi-agent path finding with kinematic constraints,'' in \emph{Proceedings of the International Conference on Automated Planning and Scheduling}, vol.~26, 2016, pp. 477--485.

\bibitem{TTP}
K.~Okumura, Y.~Tamura, and X.~D{\'e}fago, ``Time-independent planning for multiple moving agents,'' \emph{Proceedings of the AAAI Conference on Artificial Intelligence}, vol.~35, no.~13, pp. 11\,299--11\,307, 2021.

\bibitem{pie_d}
Y.~Zhang, Z.~Chen, D.~Harabor, P.~Le~Bodic, and P.~J. Stuckey, ``Concurrent planning and execution in lifelong multi-agent path finding with delay probabilities,'' in \emph{Proceedings of the AAAI Conference on Artificial Intelligence}, vol.~39, no.~22, 2025, pp. 23\,387--23\,394.

\bibitem{traffic-flow}
Z.~Chen, D.~Harabor, J.~Li, and P.~J. Stuckey, ``Traffic flow optimisation for lifelong multi-agent path finding,'' in \emph{Proceedings of the AAAI Conference on Artificial Intelligence}, vol.~38, no.~18, 2024, pp. 20\,674--20\,682.

\bibitem{a_star}
P.~E. Hart, N.~J. Nilsson, and B.~Raphael, ``A formal basis for the heuristic determination of minimum cost paths,'' \emph{IEEE Transactions on Systems Science and Cybernetics}, vol.~4, no.~2, pp. 100--107, 1968.

\bibitem{sg}
A.~Coskun, J.~M. O'Kane, and M.~Valtorta, ``Deadlock-free online plan repair in multi-robot coordination with disturbances,'' in \emph{Proceedings of the Thirty-Fourth International Florida Artificial Intelligence Research Society Conference}.\hskip 1em plus 0.5em minus 0.4em\relax Florida Online Journals, 2021.

\end{thebibliography}

\end{document}